\documentclass[twoside,11pt,preprint]{article}

\usepackage{blindtext}
\usepackage{jmlr2e}
\usepackage{mdframed}
\usepackage{amsmath}
\usepackage{amsfonts}
\usepackage{amssymb}
\usepackage[ruled,vlined]{algorithm2e}
\usepackage{hyperref}

\hypersetup{
  colorlinks=true,      
  linkcolor=blue,       
  citecolor=blue,       
  urlcolor=blue,         
  breaklinks=true
}

\newcommand{\w}{\mathbf{w}}
\newcommand{\x}{\mathbf{x}}
\newcommand{\bv}{\mathbf{v}}
\newcommand{\ty}{\tilde{y}}
\newcommand{\tY}{\tilde{\mathcal{Y}}}
\newcommand{\D}{\mathcal{D}(\tilde{\mathcal{Y}})}
\newcommand{\kl}{\mathsf{KL}}
\newcommand{\vv}{\mathbf{v}}

\newtheorem{claim}{Claim}

\usepackage{lastpage}

\ShortHeadings{Robust Online Classification}{Wu, Grama and Szpankowski}
\firstpageno{1}

\newcommand{\gammaH}{\gamma_{\scalebox{0.5}{$\mathbf{H}$}}}
\newcommand{\gammaL}{\gamma_{\scalebox{0.5}{$\mathbf{L}$}}}

\begin{document}

\title{Robust Online Classification: From Estimation to Denoising}

\author{%
 \name{Changlong Wu} \email{wuchangl@hawaii.edu}\\
 \name{Ananth Grama} \email{ayg@cs.purdue.edu} \\
 \name{Wojciech Szpankowski} \email{szpan@purdue.edu}\\
 \addr Department of Computer Science\\
Purdue University\\
West Lafayette, IN 47907, USA
}

\maketitle

\begin{abstract}%
We study online classification of features into labels with general hypothesis classes. In our setting, true labels are determined by some function within the hypothesis class, but are corrupted by \emph{unknown} stochastic noise, and the features are generated adversarially.
  Predictions are made using observed \emph{noisy} labels and noiseless features, while the performance is  measured via minimax risk when comparing against \emph{true} labels.
  The noise mechanism is modeled via a general noise \emph{kernel} that specifies, for any individual data point, a set of distributions from which the actual noisy label distribution is chosen. We show that minimax risk is tightly characterized (up to a logarithmic factor of the hypothesis class size) by the \emph{Hellinger gap} of the noisy label distributions induced by the kernel, \emph{independent} of other properties such as the means and variances of the noise. Our main technique is based on a novel reduction to an online comparison scheme of two-hypotheses, along with a new \emph{conditional} version of Le Cam-Birgé testing suitable for online settings. Our work provides the first comprehensive characterization for noisy online classification with guarantees with respect to the ground truth while addressing \emph{general} noisy observations.
\end{abstract}

\begin{keywords}%
  Online classification, noisy label, pairwise testing, Hellinger divergence
\end{keywords}

\section{Introduction}
Learning from noisy data is a fundamental problem in many machine learning applications. Noise can originate from various sources, including low-precision measurements of physical quantities, communication errors, or noise intentionally injected by methods such as differential privacy. In such cases, one typically learns by training on \emph{noisy} (or observed) data while aiming to build a model that performs well on the \emph{true} (or latent) data. This paper focuses on \emph{online learning}~\citep{shalev2014understanding} from noisy labels, where one receives noiseless, \emph{adversarially} generated features and corresponding \emph{noisy} labels sequentially, and predicts the \emph{true} labels as the data arrive. 

Online learning has been primarily studied in the \emph{agnostic} (non-realizable) setting~\citep{ben2009agnostic,rakhlin2010online,daniely2015multiclass}, where one receives the labels in their plain (noise-free) form and the prediction risk is evaluated on the \emph{observed} labels. It is typically assumed that both the features and observed labels are generated adversarially, and prediction quality is measured via the notion of \emph{regret}, which compares the actual cumulative risk incurred by the predictor with the minimal cumulative risk incurred by the best expert in a hypothesis class. While this approach is mathematically appealing, it does not adequately characterize online learning scenarios when our goal is to achieve good performance with respect to \emph{grand truth} data that may be different from the observed (noisy) ones.

This paper considers an online learning scenario that differs from classical \emph{agnostic} online learning in two aspects: (i) we assume that the noisy labels are derived from a semi-\emph{stochastic} mechanism rather than from pure adversarial selections; (ii) our predictions are evaluated on the \emph{true} labels, not \emph{noisy} observations. An example where posing a problem as above can lead to novel insights is \cite{ben2009agnostic}, as follows:
\begin{example}
\label{exm:ben09}
    Let $\mathcal{H}\subset \{0,1\}^{\mathcal{X}}$ be a finite hypothesis class. Consider the following online learning game between Nature/Adversary and Learner over a time horizon $T$. Nature fixes a ground truth $h\in \mathcal{H}$ to start the game. At each time step $t$, Nature adversarially selects feature $\x_t\in \mathcal{X}$ and reveals it to the learner. Learner makes a prediction $\hat{y}_t$ based on prior features $\x^t=\{\x_1,\cdots,\x_t\}$ and \emph{noisy} labels $\ty^{t-1}=\{\ty_1,\cdots,\ty_{t-1}\}$. Nature then selects an (unknown) noise parameter $\eta_t\in [0,\eta]$ for some given $\eta$ (known to learner), and generates
    $$\ty_t=\mathsf{Bernoulli}(\eta_t)\oplus y_t,$$
    where $\oplus$ denotes for binary addition and $y_t=h(\x_t)$ is the \emph{true} label. It was demonstrated by~\citet[Thm 15]{ben2009agnostic} that there exists predictors $\hat{y}^T$ such that
    \begin{equation}
        \label{eq:ben09}
        \sup_{h\in \mathcal{H},\x^T\in \mathcal{X}^T}\mathbb{E}\left[\sum_{t=1}^T1\{\hat{y}_t\not=h(\x_t)\}\right]\le \frac{\log|\mathcal{H}|}{1-2\sqrt{\eta(1-\eta)}}.
    \end{equation}
\end{example}
Note that, the risk bound in (\ref{eq:ben09}) is noteworthy, as the error introduced by the noise to the true labels increases linearly with $\eta T$, yet the risk bound remains \emph{independent} of the time horizon $T$. Additionally, this risk bound implies an equivalent \emph{regret} bound for the noisy labels~\citep{ben2009agnostic}, which is tighter than the adversarial \emph{agnostic} regret $O(\sqrt{T\log |\mathcal{H}|})$ across most values of $\eta$.

Despite the elegance of the result in (\ref{eq:ben09}), several issues remain: (i) the proof provided in~\cite{ben2009agnostic} is based on a somewhat involved backward induction, which does not offer satisfactory intuition on how the stochastic nature of the noise affects the risk bound; (ii) it is unclear to what extent such phenomena can occur for more general noise type beyond bounded Bernoulli noises. This paper introduces an online learning framework for modeling \emph{general} noise mechanisms. In particular, it encompasses (\ref{eq:ben09}) as a very specific instance and provides a clear and intuitive characterization of the underlying paradigm.

Formally, let $\mathcal{Y}$ be the set of (true) labels and $\tilde{\mathcal{Y}}$ be the set of noisy observations, which we assume are finite and of size $N$ and $M$, respectively. Let $\mathcal{X}$ be the feature space. We model the noise mechanism by a \emph{noise kernel} $$\mathcal{K}:\mathcal{X}\times \mathcal{Y}\rightarrow 2^{\D},$$
where $\D$ is the set of all distributions over $\tY$. That is, the kernel $\mathcal{K}$ maps each pair $(\x,y)$ to a \emph{subset} $\mathcal{Q}_y^{\x}:=\mathcal{K}(\x,y)\subset \D$ of distributions over $\tY$. Note that the noise kernel provides a compact way of modeling the noisy label distribution directly without explicitly referring to the \emph{noise}. This is more convenient for our discussion, as ultimately the statistical information is solely determined by the noisy label distributions.
 Crucially, one should not view the noise kernel as a ``constraint"; instead, it provides a formal \emph{framework} to model \emph{any} specific noise models at hand. For instance, if the learner has no prior knowledge about a noise model, the kernel set is simply the entire set of distributions.

We consider the following \emph{robust (noisy) online classification scenario}: Nature first selects $h\in \mathcal{H}$; at each time step $t$, Nature chooses (adversarially) $\x_t\in \mathcal{X}$ and reveals it to the learner; the learner then makes a prediction $\hat{y}_t$, based  on the features $\x^{t}$ and \emph{noisy} labels $\tilde{y}^{t-1}$; an \emph{adversary} then selects a distribution $\tilde{p}_t\in \mathcal{Q}_{h(\x_t)}^{\x_t}$, samples $\tilde{y}_t\sim \tilde{p}_t$ and reveals $\tilde{y}_t$ to the learner. Let $\Phi$ and $\Psi$ be the strategies of the learner and Nature/adversary, respectively. The goal of the learner is to minimize the following expected minimax \emph{risk}:
\begin{equation}
\label{eq:intro2}
\Tilde{r}_T(\mathcal{H},\mathcal{K})=\inf_{\Phi}\sup_{\Psi}\mathbb{E}\left[\sum_{t=1}^T1\{h(\x_t)\not=\hat{y}_t\}\right],
\end{equation}
where $\hat{y}_t=\Phi(\x^t,\Tilde{y}^{t-1})$ with $\Tilde{y}_t\sim \Tilde{p}_t$ and $\Tilde{p}_t\in \mathcal{Q}_{h(\x_t)}^{\x_t}$. We refer to Section~\ref{sec:pre} for more complete specifications of our formulation. Note that the adversarial selection of distribution $\tilde{p}_t$ from the kernel set $\mathcal{Q}_{h(\x_t)}^{\x_t}$ provides more flexibility for modeling scenarios when the noisy label distribution changes even with the same true label, such as Massart's noise in Example~\ref{exm:ben09}. However, we would like to point out that even for the special case $|\mathcal{Q}_y^{\x}|=1$ for all $\x,y$, the problem is still not well studied in literature, as the distribution in $\mathcal{Q}_y^{\x}$ can be quite complicated (not necessarily Bernoulli), which we address in Section~\ref{sec:size1}.

\subsection{Results and Techniques}
Our goal is to establish fundamental limits of the minimax risk as in (\ref{eq:intro2}) by providing tight lower and upper bounds across a wide range of hypothesis classes $\mathcal{H}$ and noisy kernels $\mathcal{K}$. Specifically, we show that:
\begin{theorem}[Informal]
\label{intro:thm1}
    Let $\mathcal{H}\subset \mathcal{Y}^{\mathcal{X}}$ be a finite class with $|\mathcal{Y}|=2$, $\mathcal{K}$ be any noisy kernel that satisfies $\forall \x\in \mathcal{X}$, $\forall y,y'\in \mathcal{Y}$ with $y\not=y'$,
    $$~L^2(\mathcal{Q}_y^{\x},\mathcal{Q}_{y'}^{\x})\overset{\mathsf{def}}{=}\inf_{p\in \mathcal{Q}_y^{\x},q\in \mathcal{Q}_{y'}^{\x}}\{||p-q||_2^2\}\ge \gammaL>0$$ 
    and $\mathcal{Q}_y^{\x}=\mathcal{K}(\x,y)\subset \D$ is \emph{closed} and \emph{convex}. Then
    $\tilde{r}_T(\mathcal{H},\mathcal{K})\le \frac{16\log|\mathcal{H}|}{\gammaL}.$
\end{theorem}

Intuitively, the condition in Theorem~\ref{intro:thm1} assumes that the possible noisy label distributions in $\mathcal{Q}_{y}^{\x}$ and $\mathcal{Q}_{y'}^{\x}$ are \emph{separated} under $L^2$ distance by gap $\gammaL$. For the bounded Bernoulli noise as in Example~\ref{exm:ben09}, the set $\mathcal{Q}_y^{\x}$ corresponds to Bernoulli distribution with parameters in $[0,\eta]$ if $y=0$ and in $[1-\eta,1]$ if $y=1$. Therefore, the $L^2$ gap is $2(1-2\eta)^2$, leading to
$$\tilde{r}_T(\mathcal{H},\mathcal{K})\le \frac{8\log|\mathcal{H}|}{(1-2\eta)^2}.$$ This recovers (\ref{eq:ben09}) upto a constant factor~\footnote{See also an improved upper bound in Section~\ref{sec:special} that (asymptotically) matches the constant.}. However, our result holds for \emph{any} noisy kernel whenever it exhibits a bounded gap under $L^2$ divergence.

Theorem~\ref{intro:thm1}, while intuitively appealing, does not extend directly to more general scenarios, such as \emph{multi-class} labels, high probability guarantees and constraints beyond $L^2$ gap. Our next main result is a generic (black-box) reduction from the prediction of general hypothesis classes to the pairwise comparison of two-hypothesises.
\begin{theorem}[Informal]
\label{intro:thm2}
    Let $\mathcal{H}\subset \mathcal{Y}^{\mathcal{X}}$ be any finite class and $\mathcal{K}$ be any noisy kernel. If for any \emph{pair} $h_1,h_2\in \mathcal{H}$ there exists a prediction rule achieving risk upper bound $C(\delta)$ for $\{h_1,h_2\}$ w.p. $\ge 1-\delta$. Then, there exists a predictor for the class $\mathcal{H}$, such that the w.p. $\ge 1-\delta$ the risk is upper bounded by
    $2(1+2C(\delta/(2|\mathcal{H}|))\log |\mathcal{H}|)+\log(2/\delta).$
\end{theorem}

Note that Theorem~\ref{intro:thm2} is surprising, since it demonstrates that the \emph{high probability} risk of a general hypothesis class $\mathcal{H}$ can be reduced to the risk of pairwise comparison of two-hypothesises in $\mathcal{H}$ that introduces only an extra $\log|\mathcal{H}|$ factor, \emph{regardless} how the noisy kernel $\mathcal{K}$ behaves. This effectively decouples the adversarial property of the features with the stochastic property of the noisy labels. 

To demonstrate the power of Theorem~\ref{intro:thm2}, we establish in Theorem~\ref{thm:hptest} a generalization of the Le Cam-Birgé Testing with \emph{varying} conditional marginals for handling the pairwise comparison via the \emph{Hellinger} gap, where we recall  $H^2(p,q)=\sum_{m=1}^M(\sqrt{p[m]}-\sqrt{q[m]})^2$ is the squared Hellinger distance.
In particular, together with Theorem~\ref{intro:thm2}, this leads to our third main result:

\begin{theorem}[Informal]
    \label{intro:thm3}
    Let $\mathcal{H}\subset \mathcal{Y}^{\mathcal{X}}$ be any finite class and $\mathcal{K}$ be any kernel satisfying the separation condition in Theorem~\ref{intro:thm1} with Hellinger gap $\gammaH$. Then $$\tilde{r}_T(\mathcal{H},\mathcal{K})\le O\left(\frac{\log^2|\mathcal{H}|}{\gammaH}\right).$$ Moreover, for any kernel $\mathcal{K}$ with Hellinger gap $\gammaH$ there exist class $\mathcal{H}$ such that $$\tilde{r}_T(\mathcal{H},\mathcal{K})\ge \Omega\left(\frac{\log|\mathcal{H}|}{\gammaH}\right).$$
\end{theorem}

Note that Theorem~\ref{intro:thm3} demonstrates that the \emph{Hellinger} gap is the \emph{right} characterization of the minimax risk up to at most a logarithmic factor, and the risk is independent of the size of $\mathcal{Y}$ and $\tY$. We refer to Theorem~\ref{cor2} for more formal assertions of this fact.

\paragraph{Non-uniform Gaps and Infinite Classes.} Beyond the bounded gap scenarios, we also establish \emph{tight} (upto poly-logarithmic factors) risk bounds in Proposition~\ref{prop:softgap} for cases with \emph{soft-constrained} gaps (such as the Tsybakov-type noise), and address situations where the gap parameters are \emph{unknown}, as detailed in Theorem~\ref{thm:tsyb}. Notably, the risk scales \emph{sublinearly} w.r.t. $T$, in contrast to the \emph{constant} risks as in Theorem~\ref{intro:thm1} and Theorem~\ref{intro:thm3}. In Section~\ref{sec:special}, we discuss several special, yet importance, concrete kernels where optimal risk bounds are achievable up to a \emph{constant} factor. Lastly, in Section~\ref{sec:cover}, we explore scenarios where the hypothesis class is \emph{infinite}, such as those with finite Littlestone dimensions, and relax the adversarial assumption on the features to certain stochastic assumptions, thereby accommodating broader hypothesis classes, including those with finite VC dimensions. In particular, our results imply that (see also Corollary~\ref{corld}):
\begin{theorem}[Informal]
    Let $\mathcal{H}\subset \mathcal{Y}^{\mathcal{X}}$ be a multi-class label hypothesis class with Littlestone dimension $\mathsf{Ldim}(\mathcal{H})$ and $|\mathcal{Y}|=N$, $\mathcal{K}$ be any kernel with Hellinger gap $\gammaH$. Then $\tilde{r}_{T}(\mathcal{H},\mathcal{K})\le \frac{\mathsf{Ldim}(\mathcal{H})^2\log^2(TN)}{\gammaH}$. Moreover, an $\log(N)$ dependency is necessary.
\end{theorem}

\subsection{Related Work}
Online learning with noisy data was discussed in~\cite{cesa2011online}, which specifically considers kernel-based linear functions with zero-mean and bounded variance noises. Our work differs in that we are focusing on the classification task instead of regression. Moreover, our noisy model does not require that the noise be zero-mean. To our knowledge, \cite{ben2009agnostic} is the only work that has specifically considered the classification task, but this was limited to bounded Bernoulli noise. From a technical standpoint, the reduction to online conditional probability estimation was explored in \cite{foster2021statistical} within the context of \emph{online decision making}. However, a distinguishing feature of our work is that our conditional probability estimation problem is necessarily \emph{misspecified}, as our noisy label distributions are selected \emph{adversarially} and are unknown a priori to the learner. Analogous ideas of pairwise comparison have also been considered in the differential privacy literature, such as in \cite{gopi2020locally}, but only in \emph{batch} settings. Our problem setup is further related to \emph{differentially private} conditional distribution learning, as in \cite{pmlr-v202-wu23u}, and \emph{robust hypothesis testing}, discussed in \citep[Chapter 16]{pw22}. Online conditional probability estimation has been widely studied, see \cite{rakhlin2015sequential, bilodeau2020tight, bhatt2021sequential, bilodeau2021minimax, wu2022precise, wu2022expected}. Conditional density estimation in the \emph{batch} setting has also been extensively studied, see \cite{grunwald2020fast} for KL-divergence with misspecification and \cite{efromovich2007conditional} for $L^2$ loss. Learning from noisy labels in the \emph{batch} case was discussed in \cite{natarajan2013learning} (see also the references therein) by leveraging suitably defined proxy losses. There has been a long line of research on online prediction with \emph{adversarial} observable labels in an \emph{agnostic} formulation, see \cite{lugosi-book, ben2009agnostic, rakhlin2010online, daniely2015multiclass}.

\section{Problem Formulation and Preliminaries}
\label{sec:pre}

Let $\mathcal{X}$ be a set of features (or instances), $\mathcal{Y}$ be a set of labels, and $\Tilde{\mathcal{Y}}$ be a set of \emph{noisy observations}. We assume throughout the paper that $|\mathcal{Y}|=N$ and $|\Tilde{\mathcal{Y}}|=M$ for some integers $N,M\ge 2$. We denote
$$
\mathcal{D}(\Tilde{\mathcal{Y}})=\left\{p=(p[1], \ldots, p[M])\in [0,1]^M: 
\sum_{m=1}^M p[m]=1\right\}
$$
as the set of all \emph{probability distributions} over $\Tilde{\mathcal{Y}}$. A \emph{noisy kernel} is defined as a map 
$\mathcal{K}:\mathcal{X}\times \mathcal{Y}\rightarrow 2^{\mathcal{D}(\Tilde{\mathcal{Y}})},$
where $2^{\mathcal{D}(\Tilde{\mathcal{Y}})}$ is the set of all \emph{subsets} of $\mathcal{D}(\Tilde{\mathcal{Y}})$, i.e., the kernel $\mathcal{K}$ maps each $(\x,y)\in \mathcal{X}\times \mathcal{Y}$ to a \emph{subset of distributions} $\mathcal{K}(\x,y)\subset \mathcal{D}(\Tilde{\mathcal{Y}})$. We write  $\mathcal{Q}_y^{\x}=\mathcal{K}(\x,y)$ for notational convenience.

For any $t\in [T]$, we write $\x^t=\{\x_1,\cdots,\x_t\}$, $y^t=\{y_1,\cdots,y_t\}$ and $\Tilde{y}^t=\{\Tilde{y}_1,\cdots,\Tilde{y}_t\}$. Let $\mathcal{H}\subset \mathcal{Y}^{\mathcal{X}}$ be a class of \emph{hypotheses} and $\mathcal{K}$ be a noisy kernel as defined above. We consider the following \emph{robust online classification} scenario: 
\begin{itemize}
    \item[1.] \emph{Nature} first selects some $h\in \mathcal{H}$;
    \item[2.]At time $t$, Nature adversarially selects $\x_t\in \mathcal{X}$;
    \item[3.] Learner predicts $\hat{y}_t\in \mathcal{Y}$, based on (noisy)  history observed thus far (i.e., $\x^t,\Tilde{y}^{t-1}$);
    \item[4.] An \emph{adversary} then selects $\Tilde{p}_t\in \mathcal{Q}_{h(\x_t)}^{\x_t}$, and generates a \emph{noisy} sample $\Tilde{y}_t\sim \Tilde{p}_t$.
\end{itemize}
The goal of the \emph{learner} is to minimize the \emph{cumulative error} $\sum_{t=1}^T1\{h(\x_t)\not=\hat{y}_t\}$.

Note that the cumulative error is a \emph{random variable} that depends on all the randomness associated with the game. To remove the dependency on such randomness and to assess the fundamental limits of the prediction quality, we consider the following two measures~\footnote{We assume here the selection of $\tilde{p}^T$ and $\x^T$ are oblivious to the learner's action for simplicity. This is equivalent to the adaptive case if the learner's internal randomness are independent among different time steps by a standard argument from~\citet[Lemma 4.1]{lugosi-book}.}:
\begin{definition}
\label{def:expectrisk}
    Let $\mathcal{H}\subset \mathcal{Y}^{\mathcal{X}}$ be a set of hypotheses and $\mathcal{K}:\mathcal{X}\times\mathcal{Y}\rightarrow 2^{\mathcal{D}(\Tilde{\mathcal{Y}})}$ be a noisy kernel. We denote by $\Phi$ the (possibly randomized) strategies of the \emph{learner}. The \emph{expected minimax risk} is defined as:
\begin{equation}
    \label{eq:expectrisk}
\Tilde{r}_T(\mathcal{H},\mathcal{K})=\inf_{\Phi}\sup_{h\in\mathcal{H}}
\mathbb{Q}^T_{\mathcal{K}}
\mathbb{E}_{\hat{y}^T}\left[\sum_{t=1}^T1\{h(\x_t)\not=\hat{y}_t\}\right],
    \end{equation}
    where $\hat{y}_t\sim \Phi(\x^t,\Tilde{y}^{t-1})$ and $\mathbb{Q}^T_{\mathcal{K}}$ denotes for operator
    $$\mathbb{Q}^T_{\mathcal{K}}\equiv\sup_{\x_1\in \mathcal{X}}\sup_{\tilde{p}_1\in \mathcal{Q}_{h(\x_1)}^{\x_1}}\mathbb{E}_{\tilde{y}_1\sim \tilde{p}_1}\cdots \sup_{\x_T\in \mathcal{X}}\sup_{\tilde{p}_T\in \mathcal{Q}_{h(\x_T)}^{\x_T}}\mathbb{E}_{\tilde{y}_T\sim \tilde{p}_T}.$$
\end{definition}
By \emph{skolemization}~\citep{rakhlin2010online}, the operator 
$$
\mathbb{Q}^T_{\mathcal{K}} \equiv \sup_{\psi^T}\sup_{\tilde{p}^T}\mathbb{E}_{\tilde{y}^T\sim \tilde{p}^T},
$$ 
where $\psi^T=\{\psi_1,\cdots,\psi_T\}$ runs over all functions $\psi_t:\tilde{\mathcal{Y}}^{t-1}\rightarrow \mathcal{X}$ for $t\in [T]$ and $\tilde{p}^T$ runs over all (joint) distributions over $\tilde{\mathcal{Y}}^T$ subject to the constrains that for any $t\in [T]$ and $\tilde{y}^{t-1}$ the \emph{conditional} marginal $\tilde{p}_t$ of $\tilde{p}^T$ at $\tilde{y}_t$ conditioning on $\tilde{y}^{t-1}$ satisfies $\tilde{p}_t\in \mathcal{Q}_{h(\x_t)}^{\x_t}$ for $\x_t=\psi_{t}(\tilde{y}^{t-1})$. This leads to our next definition of the \emph{high probability} minimax risk:
\begin{definition}
\label{def:highrisk}
    Let $\mathcal{H}$, $\mathcal{K}$ and $\Phi$ be as in Definition~\ref{def:expectrisk}. For any confidence parameter $\delta>0$, the \emph{high probability minimax risk} at confidence $\delta$ is defined as the minimum number $B^{\delta}(\mathcal{H},\mathcal{K})\ge 0$ such that there exists a predictor $\Phi$ satisfying: 
    \begin{equation}
        \sup_{h\in \mathcal{H},\psi^T,\tilde{p}^T}\mathrm{Pr}\left[\sum_{t=1}^T1\{h(\x_t)\not=\hat{y}_t\}\ge B^{\delta}(\mathcal{H},\mathcal{K})\right]\le \delta,
    \end{equation}
    where the selection of $\psi^T$ and $\tilde{p}^T$ are as in the discussion above with $\x_t=\psi_t(\tilde{y}^{t-1})$ and the probability is over both $\tilde{y}^T\sim \tilde{p}^T$ and $\hat{y}^T$ for $\hat{y}_t\sim \Phi(\x^t,\tilde{y}^{t-1})$.
\end{definition}
Note that the kernel map $\mathcal{K}$ is generally \emph{known} to the learner when constructing the predictor $\Phi$. However, the induced kernel sets $\mathcal{Q}_{h(\x_t)}^{\x_t}$ are not, since they depend on the \emph{unknown} ground truth classifier $h$ and \emph{adversarially} generated features $\x^T$. In certain cases, such as Theorem~\ref{thm:main2}, the kernel map $\mathcal{K}$ is also \emph{not} required to be known.

We assume, w.l.o.g., that $\mathcal{Q}_y^{\x}$s are \emph{convex} and \emph{closed} sets for all $(\x,y)$, since the adversary can select arbitrary distribution from $\mathcal{Q}_y^{\x}$s at each time step, including randomized strategies that effectively sample from a mixture (i.e., convex combination) of distributions in $\mathcal{Q}_y^{\x}$s.

Clearly, one must introduce some constraints on the kernel $\mathcal{K}$ in order to obtain meaningful results. To do so, we introduce the following \emph{well-separation} condition:
\begin{definition}
\label{def:wellsep}
    Let $L:\D\times\D\rightarrow \mathbb{R}^{\ge 0}$ be a divergence, we say a kernel $\mathcal{K}$ is \emph{well-separated} w.r.t. $L$ at scale $\gamma>0$, if 
    $\forall \x\in \mathcal{X}$, $\forall y,y'\in \mathcal{Y}$ with $y\not=y'$,~
    $$L(\mathcal{Q}_y^{\x},\mathcal{Q}_{y'}^{\x})\overset{\mathsf{def}}{=}\inf_{p\in \mathcal{Q}_y^{\x},q\in \mathcal{Q}_{y'}^{\x}}L(p,q)\ge\gamma.$$
\end{definition}

\begin{example}
\label{exp:ber}
     Let $\mathcal{Y}$ and $\tY$ be the label and noisy observation sets. We can specify for any $y\in \mathcal{Y}$ a canonical distribution $p_y\in \D$. A natural kernel would be to define
     $$\mathcal{Q}_y^{\x}=\{p\in \D:||p-p_y||_{\mathsf{TV}}\le \epsilon\}.$$
     In this case, the kernel is well-separated with the gap $\gamma$ under total variation if $$\inf_{y\not=y'\in \mathcal{Y}}||p_y-p_{y'}||_{\mathsf{TV}}\ge \gamma+2\epsilon.$$
\end{example}
\paragraph{Bregman Divergence and Exp-concavity.} We now introduce several key technical concepts and results with proofs deferred to Appendix~\ref{sec:proof2.1}. Let $\mathcal{D}(\Tilde{\mathcal{Y}})$ be the set of probability distributions over $\Tilde{\mathcal{Y}}$. A function $L: \mathcal{D}(\Tilde{\mathcal{Y}})\times \mathcal{D}(\Tilde{\mathcal{Y}}) \rightarrow \mathbb{R}^{\ge 0}$ is referred to as a \emph{divergence}. We say a divergence $L$ is a \emph{Bregman divergence} if there exists a strictly convex function $F:\mathcal{D}(\Tilde{\mathcal{Y}})\rightarrow \mathbb{R}$ such that for any $p,q\in \mathcal{D}(\Tilde{\mathcal{Y}})$, $$L(p,q)=F(p)-F(q)-(p-q)^{\mathsf{T}}\nabla F(q).$$ Note that both KL-divergence $\mathsf{KL}(p,q)=\sum_{\ty\in \tilde{\mathcal{Y}}}p[\ty]\log\frac{p[\ty]}{q[\ty]}$
and the $L^2$-divergence $L^2(p,q)=||p-q||_2^2$ are Bregman divergences~\citep[Chapter 11.2]{lugosi-book}.

\begin{proposition}
\label{prop:bregman}
    Let $P$ be a random variable over $\mathcal{D}(\Tilde{\mathcal{Y}})$ (i.e., a random variable with values in $\mathbb{R}^M$) and $L$ be a Bregman divergence. Then for any $q_1,q_2\in \mathcal{D}(\Tilde{\mathcal{Y}})$
    $$\mathbb{E}_{p\sim P}[L(p,q_1)-L(p,q_2)]=L(\mathbb{E}_{p\sim P}[p],q_1)-L(\mathbb{E}_{p\sim P}[p],q_2).$$
\end{proposition}

A function $\ell:\tY \times \D \rightarrow \mathbb{R}^{\ge 0}$ 
is refereed to as a \emph{loss} function. For instance, the \emph{log-loss} is defined as
    $\ell^{\mathsf{log}}(\ty,p)=\mathsf{KL}(e_{\ty},p)$, and the \emph{Brier loss} is defined as
    $\ell^{\mathsf{B}}(\ty,p)=L^2(e_{\ty},p)$, where $e_{\ty}$ is the probability distribution that assigns probability $1$ to  $\ty$. We say a loss $\ell$ is \emph{$\alpha$-Exp-concave} if for any $\ty\in \tY$, the function $e^{-\alpha \ell(\ty,p)}$ is concave w.r.t. $p$ for some $\alpha\in \mathbb{R}^{\ge 0}$.

\begin{proposition}
\label{prop:exp}
    The log-loss is $1$-Exp-concave and the Brier loss is $1/4$-Exp-concave.
\end{proposition}

\section{The Binary Label Case}
\label{sec:binary}
We initiate our discussion with a simpler case, where we assume the label space $\mathcal{Y}=\{0,1\}$ is binary-valued. This will provide us with an intuitive understanding of how the stochastic nature of noisy labels impacts the risk bounds.

We are now ready to state our first main result:

\begin{theorem}
\label{thm:main1}
    Let $\mathcal{H}\subset \{0,1\}^{\mathcal{X}}$ be any \emph{finite} binary valued class, $\mathcal{K}$ be any noisy kernel that is well-separated at scale $\gammaL$ w.r.t. $L^2$ divergence. Then, the \emph{expected} minimax risk, as in Definition~\ref{def:expectrisk}, is upper bounded by
    $$\tilde{r}_T(\mathcal{H},\mathcal{K})\le \frac{16\log|\mathcal{H}|}{\gammaL}.$$
\end{theorem}

\subsection{Proof of Theorem~\ref{thm:main1}}

We begin with the following simple geometry fact that is crucial for our proof.

\begin{lemma}
\label{lem:geo}
    Let $\mathcal{Q}\subset \D$ be a convex and closed set, $p$ be a point outside of $\mathcal{Q}$ with $\gamma\overset{\mathsf{def}}{=}\inf_{q\in \mathcal{Q}}L^2(p,q)$. Denote by $q^*\in \mathcal{Q}$ the (unique) point that attains $L^2(p,q^*)=\gamma$. Then for any $q\in \mathcal{Q}$, we have
    $L^2(q,p)-L^2(q,q^*)\ge L^2(p,q^*)=\gamma.$
\end{lemma}
\begin{proof}
    By the \emph{hyperplane separation theorem}, the hyperplane perpendicular to line segment $p-q^*$ at $q^*$ separates $\mathcal{Q}$ and $p$. Therefore, the degree $\theta$ of angle formed by $p-q^*-q$ is greater than $\pi/2$. By the law of cosines,
    $L^2(q,p)\ge L^2(q,q^*)+L^2(q^*,p)=L^2(q,q^*)+\gamma$.
\end{proof}

Our key idea of proving Theorem~\ref{thm:main1} is to reduce the robust (noisy) online classification problem ( to a suitable conditional distribution estimation problem, as discussed next.

\paragraph{Online conditional distribution estimation.} Let $\mathcal{F}\subset \D^{\mathcal{X}}$ be a class of functions mapping $\mathcal{X}$ to \emph{distributions} in $\D$.  \emph{Online Conditional Distribution Estimation} (OCDE) is a game between \emph{Nature} and an \emph{estimator} that follows the following protocol: (1) at each time step $t$, Nature selects some $\x_t\in \mathcal{X}$ and reveals it to the estimator; (2) the estimator then makes an estimation $\hat{p}_t\in \D$, based on $\x^t,\ty^{t-1}$; (3) Nature then selects some $\tilde{p}_t\in \D$, samples $\ty_t\sim \tilde{p}_t$ and reveals $\ty_t$ to the estimator. The goal is to find a (deterministic) estimator $\Phi$ that minimizes the \emph{regret}:
\begin{equation}
\label{eq:regret}
    \mathsf{Reg}_T(\mathcal{F},\Phi)=\sup_{f\in \mathcal{F}}\mathbb{Q}^T\left[\sum_{t=1}^TL(\tilde{p}_t,\hat{p}_t)-L(\tilde{p}_t,f(\x_t))\right],
\end{equation}
where $\hat{p}_t=\Phi(\x^t,\tilde{y}^{t-1})$, $\mathbb{Q}^T$ is the operator specified in Definition~\ref{def:expectrisk} by setting $\mathcal{Q}_y^{\x}= \mathcal{D}(\tilde{\mathcal{Y}})$ for all $\x,y$, and $L$ is any divergence. We emphasis that distributions $\tilde{p}^T$ are \emph{not} necessarily realizable by $f$ and are selected completely arbitrarily. This is the key that allows us to deal with \emph{unknown} noisy label distributions. 

We now establish the following key technical lemma:

\begin{lemma}
\label{lem:exp}
    Let $\mathcal{F}$ be any distribution-valued finite class and L be a Bregman divergence such that the induced loss $\ell(\ty,p)\overset{\mathsf{def}}{=}L(e_{\ty},p)$ is $\alpha$-Exp-concave. Then, there exists an estimator $\Phi$ (i.e., the Exponential Weight Average (EWA) algorithm), such that $$\mathsf{Reg}_T(\mathcal{F},\Phi)\le \frac{\log|\mathcal{F}|}{\alpha}.$$ Moreover, estimation $\hat{p}_t$ is a convex combination of $\{f(\x_t):f\in \mathcal{F}\}$.
\end{lemma}
We present the construction of the EWA algorithm in Appendix~\ref{sec:ewa} and the proof of Lemma~\ref{lem:exp} in Appendix~\ref{sec:prooflem34}. 

\vspace{0.2in}
\begin{proof}[Proof of Theorem~\ref{thm:main1}]
     We define the following distribution valued function class $\mathcal{F}$ using hypothesis class $\mathcal{H}$ and noisy kernel $\mathcal{K}$. For any $\x\in \mathcal{X}$, we denote by $\mathcal{Q}_0^{\x}$ and $\mathcal{Q}_1^{\x}$  the sets of noisy label distributions corresponding to labels $0$ and $1$, respectively. Since the kernel $\mathcal{K}$ is well-separated at scale $\gammaL$ under $L^2$ divergence, we have, by the \emph{hyperplane separation theorem}, that there must be $q_0^{\x}\in \mathcal{Q}_0^{\x}$ and $q_1^{\x}\in \mathcal{Q}_1^{\x}$ such that
$L^2(q_0^{\x},q_1^{\x})=L^2(\mathcal{Q}_0^{\x},\mathcal{Q}_1^{\x})\ge \gammaL.$
We now define for any $h\in \mathcal{H}$ the function $f_h$ such that $\forall \x\in \mathcal{X},~f_h(\x)=q_{h(\x)}^{\x}.$
Let $\mathcal{F}=\{f_h:h\in\mathcal{H}\}$ and $\Phi$ be the estimator in OCDE game from Lemma~\ref{lem:exp} with class $\mathcal{F}$ and $L^2$ divergence (using  $\x^T,\tilde{y}^T$ from the \emph{original} noisy classification game). Our \emph{classification} predictor is as follows:
\begin{equation}
\label{eq:predictor}
    \hat{y}_t=\arg\min_{y}\{L^2(q_y^{\x_t},\hat{p}_t):y\in \{0,1\}\}.
\end{equation}
That is, we predict the label $y$ so that $q_y^{\x_t}$ is closer to $\hat{p}_t$ under $L^2$ divergence, where $\hat{p}_t=\Phi(\x^t,\tilde{y}^{t-1})$.

Let $h^*\in \mathcal{H}$ be the underlying true classification function. We have by Lemma~\ref{lem:exp} and $1/4$-Exp-concavity of $L^2$ divergence that~\footnote{Since $\mathbb{Q}_{\mathcal{K}}^T[F(\psi^T,\tilde{y}^T)]\le \mathbb{Q}^T[F(\psi^T,\tilde{y}^T)]$ for any kernel $\mathcal{K}$ and function $F$, where $\mathbb{Q}^T$ is the \emph{unconstrained} operator in (\ref{eq:regret}).}
\begin{equation}
\label{eq:thm1}
    \mathbb{Q}^T_{\mathcal{K}}\left[\sum_{t=1}^TL^2(\tilde{p}_t,\hat{p}_t)-L^2(\tilde{p}_t,f_{h^*}(\x_t))\right]\le 4\log|\mathcal{F}|,
\end{equation}
where $\mathbb{Q}^T_{\mathcal{K}}$ is the operator in Definition~\ref{def:expectrisk}.

For any time step $t$, we denote by $y_t=h^*(\x_t)$ the true label. Since $q_y^{\x_t}\in \mathcal{Q}_y^{\x_t}$ are the elements satisfying  $L^2(q_0^{\x_t},q_1^{\x_t})=L^2(\mathcal{Q}_0^{\x_1},\mathcal{Q}_1^{\x_t})\ge\gammaL$ and $\hat{p}_t$ is a \emph{convex} combination of $q_0^{\x_t}$ and $q_1^{\x_t}$ (Lemma~\ref{lem:exp}), we have $q_{y_t}^{\x_t}$ is the closest element in $\mathcal{Q}_{y_t}^{\x_t}$ to $\hat{p}_t$ under $L^2$ divergence. Note that, we also have $\tilde{p}_t\in \mathcal{Q}_{y_t}^{\x_t}$. Invoking Lemma~\ref{lem:geo}, we find
\begin{equation}
\label{eq:difflower}
    L^2(\tilde{p}_t,\hat{p}_t)-L^2(\tilde{p}_t,q_{y_t}^{\x_t})\ge L^2(\hat{p}_t, q_{y_t}^{\x_t}).
\end{equation}
Denote $a_t=L^2(\tilde{p}_t,\hat{p}_t)-L^2(\tilde{p}_t,f_{h^*}(\x_t))$. We have, by (\ref{eq:difflower}) and $f_{h^*}(\x_t)=q_{y_t}^{\x_t}$ that $a_t\ge L^2(\hat{p}_t,f_{h^*}(\x_t))$. Therefore:
\begin{itemize}
    \item[1.] For all $t\in [T]$, $a_t\ge 0$, since $\forall p,q,~L^2(p,q)\ge 0$; 
    \item[2.] If $\hat{y}_t\not=y_t$, then $a_t\ge \gammaL/4$. This is because the event $\{\hat{y}_t\not=y_t\}$ implies that $L^2(\hat{p}_t, q_{y_t}^{\x_t})\ge L^2(\hat{p}_t, q_{1-y_t}^{\x_t})$. Hence, $L^2(\hat{p}_t, f_{h^*}(\x_t))=L^2(\hat{p}_t, q_{y_t}^{\x_t})\ge \gammaL/4$. Here, we used the following geometric fact: 
    \begin{align*}
        2\sqrt{L^2(\hat{p}_t, q_{y_t}^{\x_t})}&\ge\sqrt{L^2(\hat{p}_t, q_{y_t}^{\x_t})}+\sqrt{L^2(\hat{p}_t, q_{1-y_t}^{\x_t})}\\&=\sqrt{L^2(q_{y_t}^{\x_t},q_{1-y_t}^{\x_t})}\ge\sqrt{\gammaL}.
    \end{align*}
\end{itemize}
This implies that $\forall t\in [T],~a_t\ge \frac{\gammaL}{4}1\{\hat{y}_t\not=y_t\}$, therefore:
$$\sum_{t=1}^T1\{\hat{y}_t\not=y_t\}\le\frac{4}{\gammaL}\sum_{t=1}^TL^2(\tilde{p}_t,\hat{p}_t)-L^2(\tilde{p}_t,f_{h^*}(\x_t)).$$
The expected minimax risk now follows from  (\ref{eq:thm1}) since $|\mathcal{F}|\le |\mathcal{H}|$.
\end{proof}

Although both our proof and the one provided in~\cite{ben2009agnostic} are based on the EWA algorithm, the analysis and resulting algorithms are fundamentally different. For instance, in~\cite{ben2009agnostic}, the EWA algorithm runs over the original binary-valued class $\mathcal{H}$, whereas we run it over the \emph{distribution}-valued class $\mathcal{F}$. More importantly, our proof applies to \emph{any} noisy kernel that satisfies the well-separation condition (including cases where $|\tY| > 2$), which benefits from our \emph{geometric} interpretation of the kernels. 

Interestingly, for the specific setting investigated in~\cite{ben2009agnostic} (i.e., Example~\ref{exm:ben09}), our result yields the same order up to a constant factor, since $1-2\sqrt{\eta(1-\eta)}=\Theta((1-2\eta)^2)$ for $\eta\in [0,\frac{1}{2})$.

\begin{remark}
\label{remk:l2kl}
    Note that the selection of $L^2$ divergence plays a central rule in the proof of Theorem~\ref{thm:main1} thanks to Lemma~\ref{lem:geo}. A naive extension to the KL-divergence does not work, mainly due to the fact that if $q$ is a projection of point $p$ onto a convex set under KL-divergence, it does not necessarily imply that $q$ is the projection of any point along the line segment of $p$ and $q$. Therefore, our central argument in the proof of Theorem~\ref{thm:main1} that relates $1\{\hat{y}_t\not=y_t\}$ and $L(\tilde{p}_t,\hat{p}_t)-L(\tilde{p}_t,f_{h^*}(\x_t))$ will not go through. This can be remedied for certain special noisy kernels as discussed in Section~\ref{sec:special}.
\end{remark}

\section{Reduction to Pairwise Comparison: a Generic Approach}
\label{sec:multi}
As we demonstrated in Section~\ref{sec:binary}, the minimax risk can be upper bounded by $\frac{16\log|\mathcal{H}|}{\gammaL}$ if the kernel is uniformly separated by an $L^2$ gap $\gammaL$. However, two issues remain: (i) the proof technique is not directly generalizable to the multi-class label case. For instance, in the binary case we define a class $\mathcal{F}$ with values $q_0^{\x},q_1^{\x}$ that satisfy $L^2(q_0^{\x},q_1^{\x}) = L^2(\mathcal{Q}_0^{\x},\mathcal{Q}_1^{\x})$. However, in the multi-class case, this selection is less obvious since for any $y \in \mathcal{Y}$, the closest points in $\mathcal{Q}_y^{\x}$ to different sets $\mathcal{Q}_{y'}^{\x}$ are \emph{different}. There is no canonical way of assigning the value $f_h(\x)$; (ii) it is unclear whether the $L^2$ gap is the right information-theoretical measure for characterizing the minimax risk, compared to, for instance, the more natural $f$-divergences. This section is devoted to introduce a general approach for addressing these issues via a novel reduction to \emph{pairwise comparison} of two-hypothesises.

We first introduce the following technical concepts. Recall that our robust online classification problem is completely determined by the pair $(\mathcal{H},\mathcal{K})$ of hypothesis class $\mathcal{H}\subset \mathcal{Y}^{\mathcal{X}}$ and noisy kernel $\mathcal{K}$.
\begin{definition}
\label{def:pairtest}
    A robust online classification problem $(\mathcal{H},\mathcal{K})$ is said to be \emph{pairwise testable} with confidence $\delta>0$ and error bound $C(\delta)\ge 0$, if for any pair $h_i,h_j\in \mathcal{H}$, the sub-problem $(\{h_i,h_j\},\mathcal{K})$ admits a \emph{high probability minimax risk} $B^{\delta}(\{h_i,h_j\},\mathcal{K})\le C(\delta)$ at confidence $\delta$ (see Definition~\ref{def:highrisk}).
\end{definition}

Clearly, if $(\mathcal{H},\mathcal{K})$ admits a high probability minimax risk $B^{\delta}(\mathcal{H},\mathcal{K})$, then it is also pairwise testable with the same risk by taking $C(\delta)=B^{\delta}(\mathcal{H},\mathcal{K})$. Perhaps surprisingly, we will show in this section that the \emph{converse} holds as well up to a logarithmic factor.

Assume for now that the pair $(\mathcal{H},\mathcal{K})$ is \emph{pairwise testable} and class $\mathcal{H}=\{h_1,\cdots,h_K\}$ is finite of size $K$. Let $\Phi_{i,j}$ be the predictor for the sub-problem $(\{h_i,h_j\},\mathcal{K})$ with error bound $C(\delta/(2K))$ and confidence $\delta/(2K)>0$. Let $\x^T,\tilde{y}^T$ be any realization of problem $(\mathcal{H},\mathcal{K})$. We define, for any $h_i\in \mathcal{H}$ and $t\in [T]$, a \emph{surrogate loss} vector:
\begin{align}
\label{eq:surrloss}
    \forall j\in [K],~\bv_t^i[j]=1\{\Phi_{i,j}&(\x^t,\tilde{y}^{t-1})\not=h_i(\x_t)\text{ and }h_i(\x_t)\not=h_j(\x_t)\},
\end{align}
That is, the loss $\bv_t^i[j]=1$ if and only if $h_i(\x_t)\not=h_j(\x_t)$ \emph{and} the predictor $\Phi_{i,j}(\x^t,\tilde{y}^{t-1})$ differs from $h_i(\x_t)$. Given access to predictors $\Phi_{i,j}$s, our prediction rule for $(\mathcal{H},\mathcal{K})$ is then presented in Algorithm~\ref{alg:1}.

\begin{algorithm}[h]
\caption{Predictor via Pairwise Hypothesis Testing}\label{alg:1}
\textbf{Input}: Class $\mathcal{H}=\{h_1,\cdots,h_K\}$, testers $\Phi_{i,j}$ for $i,j\in [K]$ and error bound $C$

Set $S^1=\{1,\cdots,K\}$;
 
\For{$t=1,\cdots, T$}{
 Receive $\x_t$;

 Sampling index $\hat{k}_t$ from $S^t$ \emph{uniformly} and make prediction:
 $$\hat{y}_t=h_{\hat{k}_t}(\x_t);$$
 
 Receive noisy label $\tilde{y}_t$;

Set $S^{t+1}=\emptyset$;
 
 \For{$i\in S^t$}{
    Compute $l_t^i=\max_{j\in [K]}\sum_{r=1}^t\bv_r^i[j],$  where $\bv_t^i[j]$ is computed via $\Phi_{ij}$ as in (\ref{eq:surrloss});

    \If{$l_t^i\le C$}{
    Update $S^{t+1}=S^{t+1}\cup \{i\}$;
    }
 }
}
\end{algorithm}

At a high level, Algorithm~\ref{alg:1} tries to identify the ground truth classifier $h_{k^*}$ using the testing results of $\Phi_{i,j}$s. Note that pairwise testability implies, w.h.p., the errors made by tester $\Phi_{k,k^*}$ on $h_{k^*}$ is upper bounded by $C$ for all $k\in [K]$ simultaneously. However, for any other pair $i,j\not=k^*$, the tester $\Phi_{i,j}$ does not provide any guarantees, since the samples used to test $h_i,h_j$ originate from $h_{k^*}$ and is not \emph{realizable} for $\Phi_{i,j}$. The key technical challenge is to extract the testing results for $\Phi_{k,k^*}$ from the other irrelevant tests (i.e., $\Phi_{i,j}$ with $k^*\not\in \{i,j\}$), even when the $k^*$ is \emph{unknown}. This is resolved by our definition of $l_t^i$ in Algorithm~\ref{alg:1}, which computes for each $i$ the \emph{maximum} testing loss over all of its competitors. This ensures that, for the ground truth $k^*$, the loss $l_t^{k^*}\le C$. While for any other $i\not=k^*$, we have $l_t^i\ge \sum_{r=1}^t\vv_r^i[k^*]\ge \sum_{r=1}^t1\{h_i(\x_r)\not=h_{k^*}(\x_r)\}-C$. Therefore, any hypothesis $h_i$ for which $l_t^i>C$ cannot be the ground truth. Algorithm~\ref{alg:1} then maintains an index set \( S^t \) that eliminates all \( h_i \) for which \( l_t^i > C \), and makes prediction $\hat{y}_t=h_{\hat{k}_t}(\x_t)$ with $\hat{k}_t$ sampling \emph{uniformly} from \( S^t \). In particular, Algorithm~\ref{alg:1} enjoys the following risk bound:

\begin{theorem}
\label{thm:main2}
    Let $\mathcal{H}\subset \mathcal{Y}^{\mathcal{X}}$ be any finite hypothesis class of size $K$ and $\mathcal{K}$ be any noisy kernel. If the pair $(\mathcal{H},\mathcal{K})$ is pairwise testable with error bound $C(\delta)$ as in Definition~\ref{def:pairtest}, then for any $\delta>0$, the predictor in Algorithm~\ref{alg:1} with $C=C(\delta/(2K))$ achieves the \emph{high probability} minimax risk (Definition~\ref{def:highrisk}) upper bounded by:
    \begin{equation}
        B^{\delta}(\mathcal{H},\mathcal{K})\le 2(1+2C(\delta/(2K))\log K)+\log(2/\delta).
    \end{equation}
\end{theorem}
\begin{proof}
     Let $h_{k^*}\in \mathcal{H}$ be the underlying true classification function and $\psi^T$ be any fixed functions realizing the features $\x_t=\psi_t(\ty^{t-1})$ (see Definition~\ref{def:highrisk}). We take $C=C(\delta/2K)$ in Algorithm~\ref{alg:1}. By definition of \emph{pairwise testability} and union bound, we have w.p. $\ge 1-\delta/2$ over the randomness of $\tilde{y}^T$ and the internal randomness of $\Phi_{k,k^*}$s that for all $k\in [K]$,
    \begin{equation}
    \label{eq:main2}
        \sum_{t=1}^T1\{h_{k^*}(\x_t)\not=\Phi_{k,k^*}(\x^t,\tilde{y}^{t-1})\}\le C(\delta/(2K)).
    \end{equation}
    Note that for any other $\{i,j\}\not\ni k^*$, equation (\ref{eq:main2}) may not hold for predictor $\Phi_{i,j}$. However, our following argument relies only on the guarantees for predictors $\Phi_{k,k^*}$, which effectively makes our pairwise testing \emph{realizable}.

    We now condition on  the event defined in (\ref{eq:main2}). Let $\bv_t^k$ with $k\in [K]$ and $t\in [T]$ be the \emph{surrogate loss} vector, as defined in (\ref{eq:surrloss}). We observe the following key properties 
\begin{itemize}
    \item[1.] We have for all $t\in [T]$ that
    \begin{equation}
    \label{eq:prop1}
        \max_{j\in [K]}\sum_{r=1}^t\bv_r^{k^*}[j]\le C(\delta/(2K));
    \end{equation}
    \item[2.] For any $k\not=k^*$, we have for all $t\in [T]$:
    \begin{equation}
        \label{eq:prop2}
        \max_{j\in [K]}\sum_{r=1}^t\bv_r^k[j]\ge \left(\sum_{r=1}^t1\{h_k(\x_r)\not=h_{k^*}(\x_t)\}\right)-C(\delta/(2K)).
    \end{equation}
\end{itemize}
The first property is straightforward by the definition of $\bv_t^k$ and (\ref{eq:main2}). The second property holds since the lower bound is attained when $j=k^*$.

We now analyze the performance of Algorithm~\ref{alg:1}. By property (\ref{eq:prop1}), we know  that $k^*\in S^t$ for all $t\in [T]$, i.e., $|S^t|\ge 1$. 
Let $N_t=|S^t|$. We define for all $t\in [T]$ the \emph{potential}:
$$E_t=\sum_{k\in S^t}\max\left\{0, 2C(\delta/(2K))-\sum_{r=1}^t1\{h_k(\x_r)\not=h_{k^*}(\x_r)\}\right\}.$$
Clearly, we have $E_t\le 2C(\delta/(2K))N_t$. Let $D_t=|\{k\in S^t:h_k(\x_t)\not=h_{k^*}(\x_t)\}|$. We have:
\begin{equation}
    D_t\le N_t-N_{t+1}+E_t-E_{t+1},
\end{equation}
since for any $k\in S_t$ such that $h_k(\x_t)\not=h_{k^*}(\x_t)$, either $k$ is removed from $S^{t+1}$ (which contributes at most $N_t-N_{t+1}$) or its contribution to $E_{t+1}$ is decreased by $1$ when compared to $E_t$ (this is because by our construction of Algorithm~\ref{alg:1} and property (\ref{eq:prop2}) once the contributions of $k$ to $E_t$ equals $0$ it must be excluded from $S^{t+1}$). We have, by definition of $\hat{y}_t$, that:
\begin{equation}
    \mathbb{E}\left[1\{h_{k^*}(\x_t)\not=\hat{y}_t\}\right]=\frac{D_t}{|S^t|}\le \frac{N_t-N_{t+1}+E_t-E_{t+1}}{N_t}.
\end{equation}
By a standard argument~\citep[Thm 2]{kakade2005batch}, we have:
\begin{align*}
    \sum_{t=1}^T\frac{N_t-N_{t+1}}{N_t}&\le \sum_{t=1}^T\left(\frac{1}{N_t}+\frac{1}{N_t-1}+\cdots+\frac{1}{N_{t+1}+1}\right)\\&\le \sum_{k=1}^K\frac{1}{k}\le \log K.
\end{align*}
Moreover, we observe that
\begin{align*}
    \sum_{t=1}^T\frac{E_t-E_{t+1}}{N_t}&\overset{(a)}{\le} \frac{2C(\delta/(2K))N_1-E_2}{N_1}+\sum_{t=2}^T\frac{E_t-E_{t+1}}{N_t}\\
    &\overset{(b)}{\le} \frac{2C(\delta/(2K))(N_1-N_2)}{N_1}\\&\quad+\frac{2C(\delta/(2K))N_2-E_3}{N_2}+\sum_{t=3}^T\frac{E_t-E_{t+1}}{N_t}\\
    &\overset{(c)}{\le} 2C(\delta/(2K))\sum_{t=1}^{T}\frac{N_t-N_{t+1}}{N_t}\\&\le 2C(\delta/(2K))\log K,
\end{align*}
where $(a)$ and $(b)$ follow by $E_t\le 2C(\delta/(2K))N_t$ and $N_t\ge N_{t+1}$; $(c)$ follows by repeating the same argument for another $T-1$ steps.

Therefore, we conclude
$$\mathbb{E}\left[\sum_{t=1}^T1\{h_{k^*}(\x_t)\not=\hat{y}_t\}\right]\le (1+2C(\delta/(2K)))\log K,$$where the randomness is on the selection of $\hat{k}_t\sim S^t$. Since our selection of $\hat{k}_t$ are independent (conditioning on $S^t$) for different $t$, and the indicator is bounded by $1$ and non-negative, we can invoke Lemma~\ref{lem:tightchernoff} (second part) to obtain a high probability guarantee of confidence $\delta/2$ by introducing an extra $\log(2/\delta)$ additive term. The theorem now follows by a union bound with the event (\ref{eq:main2}).
\end{proof}
\begin{remark}
    Note that, it is \emph{not} immediately obvious that pairwise testing of two hypotheses can be converted into a general prediction rule a-priori. This is because the underlying true hypothesis is \emph{unknown}, and therefore many pairs tested do not provide any guarantees. We are able to resolve this issue due to the definition of the loss \(l_t^i\) (in Algorithm~\ref{alg:1}) for each hypothesis $i$, which considers the \emph{maximum} loss among all its competitors.
\end{remark}

    Theorem~\ref{thm:main2} provides a \emph{black box} reduction for converting any testing rule for two hypotheses into a prediction rule for a general hypothesis class $\mathcal{H}$, introducing only an additional $\log|\mathcal{H}|$ factor. This effectively decouples the adversarial properties of the features $\x^T$ from the statistical properties of the noisy labels $\ty^T$. The rest of this section is devoted to instantiating Theorem~\ref{thm:main2} into various scenarios by providing explicit pairwise testing rules.

\subsection{Pairwise-Testing via Hellinger Gap.} 
\label{sec:pairtest}

As discussed above, the risk of noisy online \emph{classification} can be reduced to the \emph{pairwise testing} $\Phi_{ij}$ of two hypotheses. However, we still need to construct the explicit pairwise testing rules. This section is devoted to providing a generic testing rule for \emph{general} kernels. 

Let $h_1, h_2$ be any two hypotheses. We may assume that $h_1(\x)\not=h_2(\x)$ for all features $\x$, since the agreed features do not impact our pairwise testing risk. We now provide a more compact characterization of the kernel $\mathcal{K}$ without explicitly referring to the feature $\x$. Following the discussion after Definition~\ref{def:expectrisk}, we can fix the feature selection rule $\psi^T$, and define the kernel by specifying the constrained sets $\mathcal{Q}_{y}^{\x_t}$ using only prior noisy labels $\ty^{t-1}$. Thus, we denote $\mathcal{Q}_i^{\ty^{t-1}}:=\mathcal{Q}^{\x_t}_{h_i(\x_t)}$, where $\x_t=\psi_t(\tilde{y}^{t-1})$ and $i\in \{1,2\}$ . 

For any $J\le T$, we denote $\mathcal{Q}_1^J$ and $\mathcal{Q}_2^J$ as the sets of all (joint) distributions over $\tY^J$ induced by the kernel for $h_1, h_2$, respectively. Equivalently, $p\in \mathcal{Q}_i^J$ if and only if for all $t \in [J]$ and $\ty^{t-1} \in \tY^{t-1}$, we have the \emph{conditional} marginal $p_{\ty_t \mid \ty^{t-1}} \in \mathcal{Q}_i^{\ty^{t-1}}$.

The pairwise testing of $h_1,h_2$ at time step $J+1$ is then equivalent to the (composite) \emph{hypothesis testing} w.r.t. sets $\mathcal{Q}_1^J$ and $\mathcal{Q}_2^J$. This is typically resolved using Le Cam-Birgé testing~\cite[Chapter 32.2]{pw22} if the distributions are of \emph{product} form. However, this does not hold for our purpose, since the distributions in $\mathcal{Q}_i^J$ can have highly correlated marginals. Our main result for addressing this issue is a \emph{conditional} version of Le Cam-Birgé testing, as stated in Theorem~\ref{thm:hptest} below. To the best of our knowledge, this conditional version is novel. 

Let us now recall the squared Hellinger divergence $H^2(\mathcal{P}, \mathcal{Q})=\inf_{p \in \mathcal{P}, q \in \mathcal{Q}} H^2(p, q)$.

\begin{theorem}[\emph{Conditional} Le Cam-Birgé Testing]
\label{thm:hptest}
    Let $\mathcal{Q}_1^J$ and $\mathcal{Q}_2^J$ be the class of distributions induced by a kernel upto time $J$ as defined above. If for all $t\in [J]$ and $\ty^{t-1}\in \tY^{t-1}$, the sets $\mathcal{Q}_1^{\ty^{t-1}},\mathcal{Q}_2^{\ty^{t-1}}$ are convex, and  $H^2(\mathcal{Q}_1^{\ty^{t-1}},\mathcal{Q}_2^{\ty^{t-1}})\ge \gamma_t$ for some $\gamma_t\ge 0$. Then, there exists a testing rule $\phi:\tY^J\rightarrow \{1,2\}$ such that~\footnote{Note that the tester $\phi$ implicitly depends on the feature selector $\psi^J$. This is not essential for our purposes, since such a dependency can be reduced to that of $\x^J$ (via a more tedious minimax analysis that considers the joint distribution over $\x^J, \tilde{y}^J$), which are observable to the tester.}
    $$\sup_{p\in \mathcal{Q}_1^J,q\in \mathcal{Q}_2^J}\left\{\mathrm{Pr}_{\ty^J\sim p}[\phi(\ty^J)\not=1]+\mathrm{Pr}_{\ty^J\sim q}[\phi(\ty^J)\not=2]\right\}\le 2\prod_{t=1}^J(1-\gamma_t/2)\le 2e^{-\frac{1}{2}\sum_{t=1}^J\gamma_t}.$$
\end{theorem}
\begin{proof}[Sketch]
    The proof requires a suitable application of the minimax theorem by expressing the testing error as a \emph{linear function} and arguing that the $\mathcal{Q}_i^J$s are convex. The error bound is then controlled by a careful application of the \emph{chain-rule} of R\'{e}nyi divergence. We defer the detailed proof to Appendix~\ref{sec:proofthm17}.
\end{proof}

Theorem~\ref{thm:hptest} immediately implies the following \emph{cumulative} risk bound:

\begin{proposition}
\label{prop:hptorisk}
    Let $\{h_1,h_2\}\subset \mathcal{Y}^{\mathcal{X}}$ and $\mathcal{K}$ be a noisy kernel. For any $t\in [T]$, we denote $\gamma_t=\inf_{\ty^{t-1}}H^2(\mathcal{Q}_{1}^{\ty^{t-1}},\mathcal{Q}_{2}^{\ty^{t-1}})$, where $\mathcal{Q}_i^{\ty^{t-1}}$ is the distribution class induced by $\mathcal{K}$ as discussed above. Then, for any $\delta>0$, the \emph{high probability} cumulative risk
    $$B^{\delta}(\{h_1,h_2\},\mathcal{K})\le \arg\min_{n}\left\{n\in \mathbb{N}:\sum_{t=1}^n\gamma_t\ge 2\log(2/\delta)\right\}.$$
\end{proposition}
\begin{proof}
    Let $n^*$ be the minimal number satisfying the RHS. If $t\le n^*$ (this can be checked at each time step $t$ using only $\x^{t}$ and $\mathcal{K}$), we predict arbitrarily. If $t\ge n^*+1$, we use the tester $\phi$ in Theorem~\ref{thm:hptest} with $J=n^*$ to produce an index $\hat{i}\in \{1,2\}$ and make the prediction $h_{\hat{i}}(\x_t)$ for \emph{all} following time steps. That is, we only use the tester at step $n^*+1$ and reuse the \emph{same} testing result for all following time steps. By Theorem~\ref{thm:hptest}, the probability of making errors after step $n^*+1$ is upper bounded by $\delta$. Therefore, the cumulative risk is upper bounded by $n^*$ with probability $\ge 1-\delta$.
\end{proof}

Instantiating to the \emph{well-separated} kernels, we arrive at:
\begin{corollary}
\label{cor:twohp}
    Let $\{h_1,h_2\}\subset \mathcal{Y}^{\mathcal{X}}$ and $\mathcal{K}$ be a well-separated kernel with gap $\gammaH$ under Hellinger distance (Definition~\ref{def:wellsep}). Then, for any $\delta\ge 0$ we have the \emph{high probability} cumulative risk
    $$B^{\delta}(\{h_1,h_2\},\mathcal{K})\le \frac{2\log(1/\delta)}{\gammaH}.$$
\end{corollary}
\begin{proof}
      Note that, for any time step $t$ such that $h_1(\x_t)\not=h_2(\x_t)$ we have the gap $\gamma_t$ in Proposition~\ref{prop:hptorisk} equals $\gammaH$. We now have the following prediction rule: for any time step $t$ such that $h_1(\x_t)=h_2(\x_t)$, we predict the agreed label; else, we predict the same way as in Proposition~\ref{prop:hptorisk}. Clearly, we only make errors for the second case. By Proposition~\ref{prop:hptorisk}, we have the number of errors is upper bounded by $\frac{2\log(1/\delta)}{\gammaH}$.
\end{proof}

\subsection{Characterization for Well-Separated Kernels} In this section, we establish \emph{matching} lower and upper bounds (up to a $\log|\mathcal{H}|$ factor) for the minimax risk of a general multi-class hypothesis class w.r.t. the \emph{Hellinger gap}, in contrast to Theorem~\ref{thm:main1} that applies only to binary label classes w.r.t. $L^2$ gap.

\begin{theorem}
\label{cor2}
    Let $\mathcal{H}\subset \mathcal{Y}^{\mathcal{X}}$ be a finite class of size $K$, and $\mathcal{K}$ be a kernel that is well-separated at scale $\gammaH$ w.r.t. Hellinger divergence. Then, the high probability minimax risk with confidence $\delta>0$ is upper bounded by
    \begin{equation}
        B^{\delta}(\mathcal{H},\mathcal{K})\le \frac{8\log(4K/\delta)\log K}{\gammaH}+\log(2/\delta).
    \end{equation}
    Moreover, for any kernel $\mathcal{K}$ such that there exist at least $\log K$ features $\x$ for which there exists $y\not=y'\in \mathcal{Y}$ we have $H^2(\mathcal{Q}_y^{\x},\mathcal{Q}_{y'}^{\x})\le \gammaH$, then there exists a class $\mathcal{H}$ of size $K$ such that
    $$\tilde{r}_T(\mathcal{H},\mathcal{K})\ge \Omega\left(\frac{\log K}{\gammaH}\right).$$
\end{theorem}
\begin{proof}
By Corollary~\ref{cor:twohp}, we know that $(\mathcal{H},\mathcal{K})$ is pairwise testable with error bound $C(\delta)=\frac{2\log(2/\delta)}{\gammaH}$. The upper bound on \emph{classification} risk then follows from Theorem~\ref{thm:main2} by noticing that $C(\delta/(2K))=\frac{2\log(4K/\delta)}{\gammaH}.$

    To prove the lower bound, we denote $\tau=\log K$ with $K=|\mathcal{H}|$, and $\x_1,\cdots,\x_{\tau}$ be $\tau$ distinct elements in $\mathcal{X}$ satisfies the condition of the theorem. We define for any $\mathbf{b}\in \{0,1\}^{\tau}$ a function $h_{\mathbf{b}}$ such that for all $i\in [\tau]$, $h_{\mathbf{b}}(\x_i)=y_i$ if $\mathbf{b}[i]=0$ and $h_{\mathbf{b}}(\x_i)=y_i'$ otherwise, where $y_i\not=y_i'\in \mathcal{Y}$ are the elements that satisfy  $\inf_{p\in \mathcal{Q}_{y_i}^{\x_i},q\in \mathcal{Q}_{y_i'}^{\x_i}}\{H^2(p,q)\}\le \gammaH$. Let $\mathcal{H}$ be the class consisting  of all such $h_{\mathbf{b}}$. Let $q_i\in \mathcal{Q}_{y_i}^{\x_i}$ and $q_i'\in \mathcal{Q}_{y_i'}^{\x_i}$ be the elements satisfying  $H^2(q_i,q_i')\le \gammaH$. We now partition the features $\x^T$ into $\tau$ epochs, each of length $T/\tau$, such that each epoch $i$ has constant feature $\x_i$. Let $\mathbf{h}$ be a random function selected  uniformly from $\mathcal{H}$. We claim that for any prediction rule $\hat{y}_t$ and any epoch $i$ we have
     \begin{equation}
     \label{eq:lower}
    \mathbb{E}_{\mathbf{h},\tilde{y}^T}\left[\sum_{t=iT/\tau-1}^{(i+1)T/\tau}1\{\mathbf{h}(\x_t)\not=\hat{y}_t\}\right]\ge \Omega\left(\frac{1}{\gammaH}\right),
     \end{equation}
     where $\tilde{y}_t\sim q_i$ if $\mathbf{h}(\x_i)=y_i$ and $\tilde{y}_t\sim q_i'$ otherwise. The theorem now follows by counting the errors for all $\tau$ epochs. 
     
     We now establish (\ref{eq:lower}) using the Le Cam's two point method. Clearly, for each epoch $i$, the prediction performance depends only on the label $\mathbf{y}_i=\mathbf{h}(\x_i)$, which is uniform over $\{y_i,y_i'\}$ and independent for different epochs by construction. For any time step $j$ during the $i$th epoch, we denote by $\tilde{y}^{j-1}$ and $\tilde{y}'^{j-1}$ the samples generated from $q_i$ and $q_i'$, respectively. By the Le Cam's two point method~\cite[Theorem 7.7]{pw22} the expected error at step $j$ is lower bounded by
     \begin{equation}
     \label{eq:lower1}
         \frac{1-\mathsf{TV}(\tilde{y}^{j-1},\tilde{y}'^{j-1})}{2}\ge \frac{1-\sqrt{H^2(\tilde{y}^{j-1},\tilde{y}'^{j-1})(1-H^2(\tilde{y}^{j-1},\tilde{y}'^{j-1})/4)}}{2}
     \end{equation}
     where the inequality follows from~\cite[Equation 7.20]{pw22}. Note that the RHS of (\ref{eq:lower1}) is \emph{monotone decreasing} w.r.t. $H^2(\tilde{y}^{j-1},\tilde{y}'^{j-1})$, since $H^2(p,q)\le 2$ for all $p,q$.

     By the \emph{tensorization} of Hellinger divergence~\cite[Equation 7.23]{pw22}, we have
     $$H^2(\tilde{y}^{j-1},\tilde{y}'^{j-1})=2-2(1-H^2(q_i,q_i')/2)^{j-1}\le 2-2(1-\gammaH/2)^{j-1},$$
     where the last inequality is implied by $H^2(q_i,q_i')\le \gammaH$. Using the fact $\log(1-x)\ge \frac{-x}{1-x}$, we have if $\gammaH\le 1$ and $j-1\le\frac{1}{\gammaH}$ then $2-2(1-\gammaH/2)^{j-1}\le 2(1-e^{-1})<2$. Therefore, the RHS of (\ref{eq:lower1}) is lower bounded by an \emph{absolute} positive constant for all $j-1\le \frac{1}{\gammaH}$, and hence the expected cumulative error will be lower bounded by $\Omega(1/\gammaH)$ during epoch $i$. This completes the proof.
\end{proof}

It is interesting to note that the bound in Theorem~\ref{cor2} is \emph{independent} to \emph{both} the size of label set $\mathcal{Y}$ and the noisy observation set $\tilde{\mathcal{Y}}$, as well as the time horizon $T$. Moreover, the dependency on the Hellinger gap $\gammaH$ is \emph{tight} upto only a logarithmic factor $\log|\mathcal{H}|$. This factor is inherent from our reduction to pairwise testing in Algorithm~\ref{alg:1} and we believe that removing it would require new techniques.

\begin{remark}
    Note that $H^2(p,q)\ge 4L^2(p,q)$ holds for any $p,q$. Thus, the Hellinger dependency of Theorem ~\ref{cor2} on $\gammaH$ is tighter than the $L^2$ dependency of Theorem~\ref{thm:main1}. Specifically, if we take $p$ being the uniform distribution over $\tY$ and $q$ being the distribution that takes half of the elements with probability mass $\frac{1+\epsilon}{M}$ and half with $\frac{1-\epsilon}{M}$. Then, $L^2(p,q)=\frac{\epsilon^2}{M}$ while $H^2(p,q)\ge \Omega(\epsilon^2)$. Therefore, the differences can grow linearly w.r.t. the size of set $\tY$.
\end{remark}

\subsection{Soft-Constrained Gaps}\label{sec:soft-gap} The well-separation condition in Theorem~\ref{thm:main1} and Theorem~\ref{cor2} requires a \emph{uniform} gap for all $\x_t$s. This may sometimes be too restrictive. We demonstrate in this section that such a ``hard" gap can be relaxed to a ``soft" gap, while still achieving sub-linear risk.

To this end, we consider a slightly relaxed adversary where we require that for some constant $A>0$ and $0\le \alpha <1$ we have the following soft-constraint holds:
\begin{equation}
\label{eq:softconst}
    \forall r\in (0,1/2],\frac{1}{T}\sum_{t=1}^T1\left\{\inf_{\ty^{t-1}\in \tY^{t-1}}\inf_{y\not=y'\in \mathcal{Y}}H^2(\mathcal{Q}_y^{\ty^{t-1}},\mathcal{Q}_{y'}^{\ty^{t-1}})\le r\right\}\le Ar^{\frac{\alpha}{1-\alpha}},
\end{equation}
where $\mathcal{Q}_y^{\ty^{t-1}}:=\mathcal{Q}_y^{\psi_t(\tilde{y}^{t-1})}$ for some fixed (unknown) feature selector $\psi^T$ as in Section~\ref{sec:pairtest}.

The following result follows similarly as Theorem~\ref{cor2}:
\begin{proposition}
\label{prop:softgap}
    We have
    $$\sup_{\mathcal{K}}\sup_{\mathcal{H}:|\mathcal{H}|\le K}\tilde{r}_T(\mathcal{H},\mathcal{K})=\tilde{\Theta}(T^{1-\alpha}),$$
    where the $\tilde{\Theta}$ hides poly-logarithmic factors w.r.t. $T$ and $K$, and $\mathcal{K}$ runs over all kernels satisfies (\ref{eq:softconst}).
\end{proposition}
\begin{proof}
   By Theorem~\ref{thm:main2}, we only need to consider the testing of two hypotheses $\{h_1,h_2\}$ to derive an upper bound. Let $\gamma$ be a parameter to be determined later. We have by (\ref{eq:softconst}) that the number of steps $t$ for which $\inf_{\ty^{t-1}\in \tY^{t-1}}\inf_{y\not=y'\in \mathcal{Y}}H^2(\mathcal{Q}_y^{\ty^{t-1}},\mathcal{Q}_{y'}^{\ty^{t-1}})\le \gamma$ is upper bounded by $A\gamma^{\frac{\alpha}{1-\alpha}}T$. We may assume, w.l.o.g., that all such steps are within the \emph{first} $A\gamma^{\frac{\alpha}{1-\alpha}}T$ time steps, since we can simply filter out such steps (using kernel map $\mathcal{K}$ and the observed features $\x_t$s) when constructing the testing rule. Note that the rest of the steps satisfy for all $\ty^{t-1}$ and $y\not=y'\in \mathcal{Y}$ that $H^2(\mathcal{Q}_y^{\ty^{t-1}},\mathcal{Q}_{y'}^{\ty^{t-1}})\ge \gamma$. By Corollary~\ref{cor:twohp}, the number of errors after step $A\gamma^{\frac{\alpha}{1-\alpha}}T$ is upper bounded by $\tilde{O}(\frac{1}{\gamma})$. Therefore, the total number of errors is upper bounded by
    $$\inf_{0\le \gamma<1/2} A\gamma^{\frac{\alpha}{1-\alpha}}T+\frac{2\log(1/\delta)}{\gamma}\le \tilde{O}(T^{1-\alpha}),$$
    where the upper bound follows by taking $\gamma=T^{-(1-\alpha)}$. 
    
    To see the lower bound, we define a kernel with the first $A\gamma^{\frac{\alpha}{1-\alpha}}T$ steps of gap $\gamma$ (to be determined) and define the remaining steps arbitrarily as long as it satisfies (\ref{eq:softconst}). By Theorem~\ref{cor2}, we have if $A\gamma^{\frac{\alpha}{1-\alpha}}T\ge \frac{\log |\mathcal{H}|}{\gamma}$, then an $\Omega(\frac{\log|\mathcal{H}|}{\gamma})$ lower bound holds. This is satisfied when taking $\gamma=\left(\frac{\log|\mathcal{H}|}{T}\right)^{1-\alpha}$, which completes the proof.
\end{proof}

\subsection{Unknown Gap Parameters.}
While our previous results provide sub-linear risk that is tight up to poly-logarithmic factors, we have assumed that full knowledge of the kernel sets $\mathcal{Q}_y^{\x_t}$s are known to the learner. In some cases, such information cannot be known completely (or only partially known). For instance, in the classical setting of \emph{Tsybakov noise} as discussed in Diakonikolas et al. (2021), the gap parameters are not assumed to be known.

To account for this issue, we introduce the following noisy kernel, analogous to the \emph{Tsybakov noise} in batch learning. For simplicity, we take $\mathcal{Y}=\tilde{\mathcal{Y}}=\{0,1\}$. Let $\ty\in \tilde{\mathcal{Y}}$, we denote $e_{\ty}$ as the distribution over $\tilde{\mathcal{Y}}$ that assigns probability $1$ on $\ty$ and denote $u$ as uniform distribution over $\tilde{\mathcal{Y}}$. For any $\x^T$, the kernel $\mathcal{K}$ satisfies $\mathcal{Q}_y^{\x_t}=\{\lambda' e_y+ (1-\lambda')u:\lambda'\ge \lambda_t\},$
subject to the condition that for some $A>0$ and $0\le \alpha<1$:
\begin{equation}
\label{eq:tsyb}
    \forall r\in (0,1/2],~\frac{1}{T}\sum_{t=1}^T1\left\{
\frac{\lambda_t}{2}\le r\right\}\le Ar^{\frac{\alpha}{1-\alpha}}.
\end{equation}
 To avoid overly technical complication, we assume that the parameters $\lambda_t$s are (obliviously) selected \emph{independent} of the noisy observation $\ty^T$. Crucially, we assume that the parameters $\lambda_t$s are \emph{unknown} to the learner. Observe that, the set $\mathcal{Q}_y^{\x_t}$ is completely determined by the parameter $\lambda_t$ and $y$, irrespectively of $\x_t$.

\begin{theorem}
\label{thm:tsyb}
    Let $\mathcal{H}\subset \{0,1\}^{\mathcal{X}}$ be any finite class and $\mathcal{K}$ be a kernel that satisfies condition (\ref{eq:tsyb}). Then, the expected minimax risk is upper bounded by:
    $$\tilde{r}_T(\mathcal{H},\mathcal{K})\le \tilde{O}(T^{\frac{2(1-\alpha)}{2-\alpha}}),$$
    where $\tilde{O}$ hides poly-logarithmic factors on $T$ and $|\mathcal{H}|$. Moreover, there exists class $\mathcal{H}$ and kernel $\mathcal{K}$ satisfies (\ref{eq:tsyb}), such that
    $$\tilde{r}_T(\mathcal{H},\mathcal{K})\ge \tilde{\Omega}(T^{\frac{2(1-\alpha)}{2-\alpha}}).$$
\end{theorem}
\begin{proof}
    The lower bound follows by the same argument as in Proposition~\ref{prop:softgap} by noticing that $H^2(\mathcal{Q}_0^{\x_t},\mathcal{Q}_1^{\x_t})=\Theta(\lambda_t^2)$ for sufficiently small $\lambda_t$. Therefore, it is sufficient to find the $\lambda$ for which $A\lambda^{\frac{\alpha}{1-\alpha}}T\ge \frac{\log|\mathcal{H}|}{\lambda^2}$. This is satisfied when $\lambda=\left(\frac{\log|\mathcal{H}|}{AT}\right)^{\frac{1-\alpha}{2-\alpha}}$. 

For the upper bound, we leverage Theorem~\ref{thm:main2} by constructing an explicit \emph{universal} pairwise testing rule. Let $h_1,h_2$ be any two hypothesises. We assume, w.l.o.g. (by relabeling), that $h_1(\x)=0$ and $h_2(\x)=1$ for all $\x$. At each time step $t$, we compute the empirical mean $\hat{\mu}_t=\frac{\ty_1+\cdots+\ty_{t-1}}{t-1}$, and predict $0$ if $\hat{\mu}_t\le \frac{1}{2}$ and predict $1$ otherwise. Let $\lambda_1,\cdots,\lambda_T$ be any configuration of the parameters.  Assume, w.l.o.g., that $h_1$ is the ground truth classifier. We have for any given $\ty^{t-1}$ the conditional expectation $\mathbb{E}[\ty_t\mid \ty^{t-1}]\le \frac{1}{2}-\frac{\lambda_t}{2}$. By the Hoeffding-Azuma inequality (Lemma~\ref{lem:azuma}), we have for all $t\in [T]$, the error probability
    $$\mathrm{Pr}\left[\hat{\mu}_t> \frac{1}{2}\right]\le e^{-(\sum_{i=1}^{t-1}\lambda_i)^2/2(t-1)}.$$
    Therefore, for any given $\delta>0$, we have by the union bound that w.p. $\ge 1-\delta$ the total number of errors made by the predictor is upper bounded by
    \begin{equation}
    \label{eq:unknownerr}
        \mathrm{err}_T= \sum_{t=1}^T1\left\{\sum_{j=1}^{t-1}\lambda_j\le\sqrt{2t\log(T/\delta)}\right\}.
    \end{equation}
    We now upper bound $\mathsf{err}_T$ using property (\ref{eq:tsyb}). Note that, for any given gap parameters $\lambda_1,\cdots,\lambda_T$, the worst configuration for $\mathsf{err}_T$ is when $\lambda_1\le \lambda_2\le \cdots\le \lambda_T$. To see this, we use the following ``switching" argument. Suppose otherwise, there exists some $j$ for which $\lambda_{j+1}<\lambda_j$. We show that by switching $\lambda_j$ and $\lambda_{j+1}$ will not decrease $\mathsf{err}_T$. This follows from the fact that the switch will not effect any time steps except step $j+1$ who will have the sum of gap parameters \emph{decreases}. We can therefore assume, w.l.o.g., that the gap parameters are monotone increasing. Now, we have by (\ref{eq:tsyb}) that for all $j\in [T]$
    $$\sum_{t=1}^T1\left\{\lambda_t\le (j/AT)^{\frac{1-\alpha}{\alpha}}\right\}\le j.$$
     This implies that for any time step $j$ we have $\lambda_j>\left(\frac{j}{AT}\right)^{\frac{1-\alpha}{\alpha}}$ since the gap parameters are monotone \emph{increasing}. Therefore, by integration approximation, we have
    $$\sum_{j=1}^n\lambda_j\ge \Omega(n^{\frac{1}{\alpha}}T^{-\frac{1-\alpha}{\alpha}}).$$
    Setting \( n^{\frac{1}{\alpha}}T^{-\frac{1-\alpha}{\alpha}} \le n^{\frac{1}{2}} \cdot \sqrt{2\log(T/\delta)} \), we find that \( n = \tilde{O}(T^{\frac{2(1-\alpha)}{2-\alpha}}) \). This implies that for any time step \( t \ge n \), the \( t \)'th indicator in (\ref{eq:unknownerr}) equals 0. Therefore, the risk of pairwise testing is upper bound by $\mathsf{err}_T\le \tilde{O}(T^{\frac{2(1-\alpha)}{2-\alpha}})$ w.p. $\ge 1-\delta$, where $\tilde{O}$ hides the factor $\log(T/\delta)$. The upper bound of the theorem now follows by Theorem~\ref{thm:main2}.
\end{proof}
\begin{remark}
Observe that the lower and upper bounds of Theorem~\ref{thm:tsyb} \emph{match} up to poly-logarithmic factors w.r.t. $T$ and $|\mathcal{H}|$. Moreover, the proof technique for the upper bound can be generalized to the case when $\mathcal{Q}_0^{\x}$ encompasses \emph{any} distributions over $[0,1]$ with means in $[0,\frac{1-\lambda_t}{2}]$ (and in $[\frac{1+\lambda_t}{2},1]$ for $\mathcal{Q}_1^{\x}$), not only for Bernoulli distributions as in (\ref{eq:tsyb}).
\end{remark}

    Note that, the pairwise testing rule derived in the proof of Theorem~\ref{thm:tsyb} requires no information about the underlying distributions. This differs from the general testing rule derived from Theorem~\ref{thm:hptest}, which requires the likelihood ratio of distributions $p_1^* \in \mathcal{Q}_1^J$ and $p_2^* \in \mathcal{Q}_2^J$ that achieve $||p_1^* - p_2^*||{\mathsf{TV}} = \mathsf{TV}(\mathcal{Q}_1^J, \mathcal{Q}_2^J)$ (see Appendix~\ref{sec:proofthm17}).

\section{Tight Bounds via Log-loss}
\label{sec:special}

In this section, we introduce a refine technique based on the reduction to \emph{online conditional distribution estimation} as discussed in Section~\ref{sec:binary}. We shall use again Lemma~\ref{lem:exp} but with the \emph{log-loss}. This yields tight risk dependency on \emph{both} $\log|\mathcal{H}|$ and the gap parameter for certain special, yet important, noisy kernels.

\subsection{The Randomized Response Mechanism}
 Let $\mathcal{Y}=\tilde{\mathcal{Y}}=\{1,\cdots,M\}$. We denote by $u$ the \emph{uniform} distribution over $\tilde{\mathcal{Y}}$ and $e_{\ty}$ the distribution that assign probability $1$ on $\ty\in\tY$. For any $\eta>0$, we define a \emph{homogeneous} (i.e., independent of $\x$) kernel:
$$\forall \x\in \mathcal{X},~y\in \mathcal{Y},~\mathcal{K}^{\eta}(\x,y)=\{(1-\eta') e_y+\eta' u:\eta'\in [0,\eta)\}.$$
Note that, this kernel can be interpreted as the \emph{randomized response mechanism} with multiple outcomes in differential privacy~\citep{dwork2014algorithmic}, where $\eta$ is interpreted as the noise level of \emph{perturbing} the true labels. For instance, it achieves $(\epsilon,0)$-local differential privacy if we set $\eta=\frac{M}{e^{\epsilon}-1+M}$.

\begin{theorem}
\label{thm:symetric}
    Let $\mathcal{H}\subset \mathcal{Y}^{\mathcal{X}}$ be any finite class and $\mathcal{K}^{\eta}$ be as defined above with $0\le \eta<1$. Then, the expected minimax risk is upper bounded by
    $$\tilde{r}_T(\mathcal{H},\mathcal{K}^{\eta})\le \frac{\log|\mathcal{H}|}{(1-\eta)^2/2}.$$
    Moreover, the \emph{high probability} minimax risk at confidence $\delta>0$ is upper bounded by
    $$B^{\delta}(\mathcal{H},\mathcal{K}^{\eta})\le\frac{\log|\mathcal{H}|+2\log(1/\delta)}{(1-\eta)^2/4}.$$
    Furthermore, for $1-\eta\ll \frac{1}{M}$ we have $B^{\delta}(\mathcal{H},\mathcal{K}^{\eta})\le O\left(\frac{\log|\mathcal{H}|+\log(1/\delta)}{M(1-\eta)^2}\right)$. 
\end{theorem}
\begin{proof}
    Our proof follows a similar path as the proof of Theorem~\ref{thm:main1}. For any $h\in \mathcal{H}$, we define a distribution-valued function $f_h$ such that $f_h(\x)=(1-\eta)e_{h(\x)}+\eta u$. Let $\mathcal{F}=\{f_h:h\in \mathcal{H}\}$. Invoking Lemma~\ref{lem:exp} with log-loss and using the fact the KL-divergence is Bregman and $1$-Exp-concave, there exists estimators $\hat{p}^T$ such that
    $$
    \sup_{f\in \mathcal{F}}\mathbb{Q}^T_{\mathcal{K}}
    \left[\sum_{t=1}^T\kl(\tilde{p}_t,\hat{p}_t)-\kl(\tilde{p}_t,f(\x_t))\right]\le \log|\mathcal{H}|,
    $$
where $\mathbb{Q}^T_{\mathcal{K}}$ is the operator in Definition~\ref{def:expectrisk}. We now define the following classifier:
    $$\hat{y}_t=\arg\max_{y}\{\hat{p}_t[y]:y\in \mathcal{Y}\}.$$
    Note that, this is a \emph{multi-class} classifier. Let now $h^*\in \mathcal{H}$ be the underlying true classification function and $\tilde{p}^T$ be the noisy label distributions selected by the adversary. We have:
    \begin{claim}
    \label{claim1}
    The following holds for all $t\le T$:
    $$\mathsf{KL}(\tilde{p}_t,\hat{p}_t)-\mathsf{KL}(\tilde{p}_t,f_{h^*}(\x_t))\ge 0.$$
    Moreover, if $\hat{y}_t\not=h^*(\x_t)$ then
    $$\mathsf{KL}(\tilde{p}_t,\hat{p}_t)-\mathsf{KL}(\tilde{p}_t,f_{h^*}(\x_t))\ge (1-\eta)^2/2.$$
\end{claim}
\begin{proof}[Proof of the Claim]
Let $y_t=h^*(\x_t)$ and  $e_t\in \D$ be the distribution that assigns probability $1$ on $y_t$. By the definition $f_{h^*}(\x_t)=\lambda e_t+(1-\lambda)u$ and $\tilde{p}_t=\lambda_t e_t+(1-\lambda_t)u$, where $\lambda=1-\eta$ and $\lambda_t=1-\eta_t$ for some $\eta_t\le \eta$. Since $0\le \eta_t\le \eta$, we have $1\ge \lambda_t\ge \lambda$. Note that, $\mathsf{KL}(\tilde{p}_t,\hat{p}_t)-\mathsf{KL}(\tilde{p}_t,f_{h^*}(\x_t))$ is a linear function w.r.t. $\lambda_t$ (Proposition~\ref{prop:bregman}), and  it takes the minimal value at $\lambda_t\in \{1,\lambda\}$ and therefore
$$\mathsf{KL}(\tilde{p}_t,\hat{p}_t)-\mathsf{KL}(\tilde{p}_t,f_{h^*}(\x_t))\ge \min\{\log(f_{h^*}(\x_t)[y_t]/\hat{p}_t[y_t]),\mathsf{KL}(f_{h^*}(\x_t),\hat{p}_t)\}.$$
Clearly, the second KL-divergence term is positive. We now show that $\log(f_{h^*}(\x_t)[y_t]/\hat{p}_t[y_t])\ge 0$. To see this, we have by Lemma~\ref{lem:exp} that $\hat{p}_t$ is a \emph{convex} combination of $\{f(\x_t):f\in \mathcal{F}\}$ and therefore $\hat{p}_t=\lambda a_t+(1-\lambda)u$ for some $a_t\in \D$. This implies that $\hat{p}_t[y_t]=\lambda a_t[y_t]+(1-\lambda)\frac{1}{M}$ and $f_{h^*}(\x_t)[y_t]=\lambda +(1-\lambda)\frac{1}{M}$. Since $a_t[y_t]\le 1$, we have $f_{h^*}(\x_t)[y_t]\ge \hat{p}_t[y_t]$. The first part of the claim now follows.

We now prove the second part of the claim. Note that in order for $\hat{y}_t\not=y_t$ we must have $a_t[y_t]\le \frac{1}{2}$, since $\hat{y}_t$ is defined to be the label with maximum probability mass under $\hat{p}_t$. Therefore, $$\log(f_{h^*}(\x_t)[y_t]/\hat{p}_t[y_t])\ge \log\left(\frac{\lambda+(1-\lambda)/M}{\lambda/2+(1-\lambda)/M}\right)=\log\left(1+\frac{\lambda/2}{\lambda/2+(1-\lambda)/M}\right)\ge \log(1+\lambda)$$
where the second inequality follows from $\lambda/2+(1-\lambda)/M\le 1/2$. Furthermore, we have
$$\mathsf{KL}(f_{h^*}(\x_t),\hat{p}_t)\ge \frac{1}{2}||f_{h^*}(\x_t)-\hat{p}_t||_1^2\ge \lambda^2/2,$$
where the first inequality is a consequence of  Pinsker's inequality~\citep{pw22} and the second inequality follows by $||f_{h^*}(\x_t)-\hat{p}_t||_1=\lambda||e_{y_t}-a_t||_1=\lambda (2|1-a_t[y_t]|)\ge \lambda$, since $a_t[y_t]\le \frac{1}{2}$. The claim now follows by that $\log(1+\lambda)\ge \lambda^2/2$ for all $0\le \lambda\le 1$.
\end{proof}

The first part of the theorem now follows by the same argument as in the proof of Theorem~\ref{thm:main1}. The proof of the second and third parts requires a careful analysis for relating the log-loss with the Hellinger distance and employing a martingale concentration inequality similar to~\cite[Lemma A.14]{foster2021statistical}. We defer the technical proof to Appendix~\ref{sec:apphigh} for readability.
\end{proof}

To complement the upper bounds of Theorem~\ref{thm:symetric}, we have the following matching lower bound follows directly from Theorem~\ref{cor2}:

\begin{corollary}
    There exists a class $\mathcal{H}$ such that for $1-\eta\ll \frac{1}{M}$ we have
    $$\tilde{r}(\mathcal{H},\mathcal{K}^{\eta})\ge \Omega\left(\frac{\log|\mathcal{H}|}{M(1-\eta)^2}\right).$$
\end{corollary}
\begin{proof}
    Specializing to the setting in Theorem~\ref{cor2}, we know that the squared Helliger gap is of order
         $$\left(\sqrt{\frac{\eta}{M}}-\sqrt{1-\frac{(M-1)\eta}{M}}\right)^2\sim \frac{M(1-\eta)^2}{4},$$
         when $1-\eta\ll \frac{1}{M}$ (by Taylor expansion). This implies an $\Omega\left(\frac{\log|\mathcal{H}|}{M(1-\eta)^2}\right)$ lower bound.
\end{proof}

\begin{remark}
    Taking \(\eta=\frac{M}{e^{\epsilon}-1+M}\) for sufficiently small \(\epsilon\), we have
\[
\tilde{r}_T(\mathcal{H},\mathcal{K}^{\eta})=\Theta\left(\frac{M\log|\mathcal{H}|}{\epsilon^2}\right),
\]
and the randomized response mechanism with kernel \(\mathcal{K}^{\eta}\) achieves \((\epsilon,0)\)-local differential privacy. This holds even when the noisy parameters used by different local parties vary, as long as they are upper bounded by \(\eta\).
\end{remark}
\subsection{Kernel Set of Size One}
\label{sec:size1}
In this section, we establish an upper bound for the special case when the kernel set size $|\mathcal{Q}_y^{\x}|=1$ for all $\x,y$. This matches the lower bound in Theorem~\ref{cor2} up to a \emph{constant} factor.
\begin{theorem}
\label{thm:oneelemt}
    Let $\mathcal{H}\subset \mathcal{Y}^{\mathcal{X}}$ be any finite class and $\mathcal{K}$ be any noisy kernel that is well-separated at scale $\gammaH$ w.r.t. squared Hellinger distance such that $|\mathcal{Q}_y^{\x}|=1$ for all $\x,y$. Then the high probability minimax risk at confidence $\delta>0$ is upper bounded by
    $$B^{\delta}(\mathcal{H},\mathcal{K})\le O\left(\frac{\log(|\mathcal{H}|/\delta)}{\gammaH}\right).$$
\end{theorem}

\begin{proof}
    Our proof follows a similar path as in the proof of Theorem~\ref{thm:main1}, but replacing the $L^2$ loss with log-loss. Specifically, for any $h\in \mathcal{H}$, we define $f_h(\x)=q_{h(\x)}^{\x}$, where $q_{h(\x)}^{\x}$ is the unique element in $\mathcal{Q}_{h(\x)}^{\x}$. Denote $\mathcal{F}=\{f_h:h\in \mathcal{H}\}$. We run the EWA algorithm (Algorithm~\ref{alg:2}) over $\mathcal{F}$ with $\alpha=1$ and $\ell$ being the log-loss, and produce an estimator $\hat{p}^T$. The classifier is then given by
    $$\hat{y}_t=\arg\min_{y\in \mathcal{Y}}\{H^2(q_{y}^{\x_t},\hat{p}_t)\}.$$
    Now, our key observation is that the noisy label distribution $\tilde{p}_t=f_{h^*}(\x_t)$ is \emph{well-specified} (since $|\mathcal{Q}_y^{\x}|=1$, the only choice for $\tilde{p}_t$ is $f_{h^*(\x_t)}$), where $h^*$ is the ground truth classifier. Therefore, invoking~\cite[Lemma A.14]{foster2021statistical}, we find
    $$\mathrm{Pr}\left[\sum_{t=1}^TH^2(\tilde{p}_t,\hat{p}_t)\le \log|\mathcal{F}|+2\log(1/\delta)\right]\ge 1-\delta.$$
    We claim that $1\{\hat{y}_t\not=h^*(\x_t)\}\le \frac{4}{\gammaH}H^2(\tilde{p}_t,\hat{p}_t)$. Clearly, this automatically satisfies if $\hat{y}_t=h^*(\x_t)$. For $\hat{y}_t\not=h^*(\x_t)$, we have $H^2(q_{\hat{y}_t}^{\x_t},\hat{p}_t)\le H^2(q_{h^*(\x_t)}^{\x_t},\hat{p}_t)=H^2(\tilde{p}_t,\hat{p}_t)$ by definition of $\hat{y}_t$. This implies that
    $$H^2(\tilde{p},\hat{p}_t)\ge \frac{1}{4}H^2(q_{\hat{y}_t}^{\x_t},q_{h^*(\x_t)}^{\x_t})\ge \frac{\gammaH}{4},$$
    where the first inequality follows by triangle inequality of Hellinger distance (the factor $\frac{1}{4}$ comes from the conversion form squared Hellinger distance to Hellinger distance), 
and the second inequality follows by definition of $\gammaH$. Therefore, we have w.p. $\ge 1-\delta$ that
$$\sum_{t=1}^T1\{\hat{y}_t\not=h^*(\x_t)\}\le \frac{4}{\gammaH}(\log|\mathcal{F}|+2\log(1/\delta)).$$
This completes the proof since $|\mathcal{H}|\ge |\mathcal{F}|$.
\end{proof}

Observe that the key ingredient in the proof of Theorem~\ref{thm:oneelemt} is the realizability of $\tilde{p}_t$ by $f_{h^*}$ due to the property $|\mathcal{Q}_y^{\x}|=1$, which does not hold for general kernels.

\section{Extensions for Stochastically Generated Features}
\label{sec:cover}

We have demonstrated in the previous sections that the minimax risk of our robust online classification problem can be effectively bounded for a finite hypothesis class $\mathcal{H}$ and adversarially generated features $\x^T$. We now demonstrate how such result can be generalized to infinite classes and general {\it stochastic} feature generating processes via suitable covering of the class.

We first introduce the following notion of covering from~\cite{wu2023online}, which generalizes a similar concept in~\cite{ben2009agnostic}.
\begin{definition}
\label{def:seqcover}
    Let $\mathcal{H}\subset \mathcal{Y}^{\mathcal{X}}$ be any hypothesis class and $\mathsf{P}$ be any class of random processes over $\mathcal{X}^T$. We say a class of functions $\mathcal{G}\subset \mathcal{Y}^{\mathcal{X}^*}$ (where $\mathcal{X}^*$ is a set of all finite sequences of $\mathcal{X}$) is a \emph{stochastic sequential covering} of $\mathcal{H}$ w.r.t. $\mathsf{P}$ at scale $0$ and confidence $\delta$ if
    $$\forall \pmb{\nu}^T\in \mathsf{P},~\mathrm{Pr}_{\x^T\sim \pmb{\nu}^T}\left[\exists h\in \mathcal{H}\forall g\in \mathcal{G}\exists t\in [T],~h(\x_t)\not=g(\x^t)\right]\le \delta.$$
\end{definition}

Observe that the (adversarial) sequential experts as constructed in~\cite[Section 3.1]{ben2009agnostic} can be viewed as a \emph{stochastic} sequential covering in Definition~\ref{def:seqcover} with the distribution class $\mathsf{P}$ consists of all \emph{singleton} distributions over $\mathcal{X}^T$ (i.e., distributions that assign probability $1$ on a single sequence $\x^T$).

\paragraph{Infinite Classes.} We now have the following result that reduces the minimax risk of an infinite class to the size of the stochastic sequential cover.

\begin{theorem}
\label{thm:finite2infinite}
    Let $\mathcal{H}\subset \mathcal{Y}^{\mathcal{X}}$ be any hypothesis class, $\mathsf{P}$ be any class of random processes over $\mathcal{X}^T$ and $\mathcal{K}$ be a noisy kernel that is well-separated w.r.t. Hellinger divergence at scale $\gammaH$. If there exists a finite stochastic sequential cover $\mathcal{G}\subset \mathcal{Y}^{\mathcal{X}^*}$ of $\mathcal{H}$ w.r.t. $\mathsf{P}$ at scale $0$ and confidence $\delta/2>0$, then there exists a predictor such that for all $\pmb{\nu}^T\in \mathsf{P}$, if $\x^T\sim \pmb{\nu}^T$ then w.p. $\ge 1-\delta$ over all randomness involved, the risk is upper bounded by
    $$O\left(\frac{\log(|\mathcal{G}|)\log(4|\mathcal{G}|/\delta)}{\gammaH}\right).$$
\end{theorem}
\begin{proof}
    Let $A$ be the event over $\x^T$ so that $\forall h\in \mathcal{H}$, $\exists g\in \mathcal{G}$ such that
    $\forall t\in [T],~h(\x_t)=g(\x^t).$
    Let now $\pmb{\nu}^T\in \mathsf{P}$ be the underlying true feature generating process. We have by the definition of stochastic sequential covering that $\mathrm{Pr}_{\x^T}[A]\ge 1-\delta/2$. We now observe that  Theorem~\ref{cor2} holds for sequential functions as well. Therefore, taking confidence parameter $\delta/2$, the prediction rule derived from Theorem~\ref{cor2} w.r.t. class $\mathcal{G}$ yields high probability minimax risk upper bounded by
\begin{equation}
    \label{eq-b1}
 O\left(\frac{\log(|\mathcal{G}|)\log(4|\mathcal{G}|/\delta )}{\gammaH}\right).
\end{equation}
Let $h^*\in \mathcal{H}$ be the underlying true function, $\x^T\in A$ be any realization of the feature, and $g^*$ be the sequential covering function of $h^*$ at scale $0$. Note that, $g^*$ has the same labeling as $h^*$ on $\x^T$. Therefore, any predictor has the same behaviours when running on $h^*$ and $g^*$, and thus the high probability minimax risk for $\mathcal{H}$ is upper bounded by that of $\mathcal{G}$.
The theorem now follows by a union bound.
\end{proof}

    Note that, any bounds that we have established in the previous sections for finite class can be extended to the infinite classes; these bounds depend only on the stochastic sequential cover size using a similar argument as in Theorem~\ref{thm:finite2infinite}. We will not discuss  all such cases in this paper in the interest of clarity of presentation. As a demonstration, we establish the following concrete minimax risk bounds:

\begin{corollary}
\label{corld}
    Let $\mathcal{H}\subset \mathcal{Y}^{\mathcal{X}}$ be a class with finite Littlestone dimension $\mathsf{Ldim}(\mathcal{H})$~\citep{daniely2015multiclass} and $|\mathcal{Y}|=N$. If the features are generated adversarially, and $\mathcal{K}$ is any noisy kernel that is well-separated w.r.t. Hellinger divergence at scale $\gammaH$. Then, the high probability minimax risk at confidence $\delta$ is upper bounded by
    $$B^{\delta}(\mathcal{H},\mathcal{K})\le O\left(\frac{\mathsf{Ldim}(\mathcal{H})^2\log^2(TN)+\mathsf{Ldim}(\mathcal{H})\log(4TN/\delta)}{\gammaH}\right).$$
    Moreover, for the noisy kernel $\mathcal{K}^{\eta}$
    as in Theorem~\ref{thm:symetric}, the high probability minimax risk with confidence $\delta>0$ is upper bounded by
    $$B^{\delta}(\mathcal{H},\mathcal{K}^{\eta})\le \frac{(\mathsf{Ldim}(\mathcal{H})+1)\log(TN)+2\log(1/\delta)}{(1-\eta)^2/4}.$$
\end{corollary}
\begin{proof}
    The first part follows directly from Theorem~\ref{thm:finite2infinite} and the fact that the sequential covering of $\mathcal{H}$ w.r.t. adversarial selection of $\mathcal{X}^T$ is of order $(TN)^{\mathsf{Ldim}(\mathcal{H})+1}$ by~\cite[Theorem 25]{daniely2015multiclass}. The second part follows by Theorem~\ref{thm:symetric}.
\end{proof}

\begin{remark}
    Note that the logarithmic dependency on $N$ in Corollary~\ref{corld} is necessary for finite Littlestone dimensional classes. To see this, consider the class $\mathcal{H}=\{h_y(\x)=y:y\in \mathcal{Y}\}$ of constant functions, which has Littlestone dimension $1$. However, one may define a kernel that assigns each $y\in \mathcal{Y}$ a distribution $p_y$ such that $\mathsf{KL}(p_y,p_{y'})\le O(\epsilon)$ and $H^2(p_y,p_{y'})\ge \Omega(\epsilon)$ for all distinct $y,y'\in \mathcal{Y}$ (by taking, e.g., distributions of the form $p_y[\tilde{y}]=\frac{1\pm \epsilon}{M}$ where $M=|\tY|\sim \log N$). An $\Omega(\log N)$ risk lower bound then follows from Fano's inequality. This contrasts with the noiseless case, where the regret is independent of the label set size $N$.
\end{remark}

\paragraph{$\sigma$-Smoothed Processes.} Finally, we apply our results for a large class of distributions over $\mathcal{X}^T$ known as $\sigma$-smoothed processes. For any given distribution $\mu$ over $\mathcal{X}$, we say a distribution $\nu$ over $\mathcal{X}$ is $\sigma$-smooth w.r.t. $\mu$ if for all measurable sets $A\subset \mathcal{X}$, we have $\nu(A)\le \mu(A)/\sigma$~\citep{haghtalab2020smoothed}. A random process $\pmb{\nu}^T$ over $\mathcal{X}^T$ is said to be $\sigma$-smooth if the \emph{conditional marginal} $\pmb{\nu}^T(\cdot\mid X^{t-1})$ is $\sigma$-smooth w.r.t. $\mu$ for all $t\le T$, almost surely. For instance, if $\sigma=1$, we reduce to the $i.i.d.$ process case.
\begin{corollary}
    Let $\mathcal{H}\subset \mathcal{Y}^{\mathcal{X}}$ be a class with finite VC-dimension $\mathsf{VC}(\mathcal{H})$ and $|\mathcal{Y}|=2$, $\mathsf{S}^{\sigma}(\mu)$ be the class of all $\sigma$-smoothed processes w.r.t. $\mu$, and $\mathcal{K}^{\eta}$ be the noisy kernel as in Theorem~\ref{thm:symetric}. Then for any $\pmb{\nu}^T\in \mathsf{S}^{\sigma}(\mu)$, if $\x^T\sim \pmb{\nu}^T$ then the \emph{high probability} minimax risk at confidence $\delta>0$ is upper bounded by
    $$O\left(\frac{\mathsf{VC}(\mathcal{H})\log(T/\sigma)+\log(1/\delta)}{(1-\eta)^2}\right).$$
\end{corollary}
\begin{proof}
    By~\cite[Proposition 22]{wu2023online}, $\mathcal{H}$ admits an stochastic sequential cover $\mathcal{G}$ at confidence $\delta/2> 0$ such that 
    $$\log|\mathcal{G}|\le O(\mathsf{VC}(\mathcal{H})\log (T/\sigma)+\log(1/\delta)).$$ We now condition on the event of the exact covering. By Theorem~\ref{thm:symetric} (second part), the high probability minimax risk at confidence $\delta/2$ is upper bounded by
    $$ O\left(\frac{\log|G|+\log(2/\delta)}{(1-\eta)^2}\right)\le O\left(\frac{\mathsf{VC}(\mathcal{H})\log(T/\sigma)+\log(1/\delta)}{(1-\eta)^2}\right).$$
    The result now follows by a union bound.
\end{proof}
\begin{remark}
    We refer to~\cite{wu2023online} for more results on the stochastic sequential covering estimates of various feature distribution classes, including the cases when the reference measure $\mu$ is unknown.
\end{remark}

\section{Conclusion}
In this paper, we provide nearly matching lower and upper bounds for online classification with noisy labels via the Hellinger gap of the induced noisy label distributions. Our approach is effective for a wide range of hypothesis classes and noisy mechanisms. We expect our results to have broad applications, such as in online learning under (local) differential privacy constraints and online denoising tasks involving data derived from (noisy) physical measurements, such as learning from quantum data. The main open problem remaining is to close the logarithmic gap in Theorem~\ref{cor2} for \emph{general} kernels. While our work primarily focuses on the information-theoretically achievable minimax risks, we believe that finding computationally efficient predictors (including oracle-efficient methods, as in \cite{kakade2005batch}) would also be of significant interest.

\newpage
\appendix

\section{Martingale Concentration Inequalities}
In this appendix, we present some standard concentration results for martingales, which will be useful for deriving high probability guarantees. We refer to~\cite[Chapter 13.1]{zhang2023mathematical} for the proofs.

\begin{lemma}[Azuma's Inequality]
\label{lem:azuma}
    Let $X_1,\cdots,X_T$ be an arbitrary random process adaptive to some filtration $\{\mathcal{F}_t\}_{t\le T}$ such that $|X_t|\le M$ for all $t\le T$. Let $Y_t=\mathbb{E}[X_t\mid \mathcal{F}_{t-1}]$ be the conditional expected random variable of $X_t$. Then for all $\delta>0$, we have
    $$\mathrm{Pr}\left[\sum_{t=1}^TY_t<\sum_{t=1}^TX_t+M\sqrt{(T/2)\log(1/\delta)}\right]\ge 1-\delta,$$
    and
    $$\mathrm{Pr}\left[\sum_{t=1}^TY_t>\sum_{t=1}^TX_t-M\sqrt{(T/2)\log(1/\delta)}\right]\ge 1-\delta.$$
\end{lemma}

The following lemma provides a tighter concentration when $X_t\ge 0$, which can be viewed as an Martingale version of the multiplicative Chernoff bound.

\begin{lemma}[{\cite[Theorem 13.5]{zhang2023mathematical}}]
\label{lem:tightchernoff}
    Let $X_1,\cdots,X_T$ be an arbitrary random process adaptive to some filtration $\{\mathcal{F}_t\}_{t\le T}$ such that $0\le X_t\le M$ for all $t\le T$. Let $Y_t=\mathbb{E}[X_t\mid \mathcal{F}_{t-1}]$ be the conditional expected random variable of $X_t$. Then for all $\delta>0$ we have
    $$\mathrm{Pr}\left[\sum_{t=1}^TY_t<2\sum_{t=1}^T X_t + 2M\log(1/\delta)\right]\ge 1-\delta,$$
    and
    $$\mathrm{Pr}\left[\sum_{t=1}^T Y_t > \frac{1}{2}\sum_{t=1}^TX_t - (M/2)\log(1/\delta)\right]\ge 1-\delta.$$
\end{lemma}
\begin{proof}
    Applying~\citet[Thm 13.5]{zhang2023mathematical} with $\xi_t=X_t/M$ and $\lambda=1$ in the theorem.
\end{proof}

Finally, we quote one more large deviations result that we need in Appendix~\ref{sec:apphigh}.

\begin{lemma}[{\cite[Theorem 13.2]{zhang2023mathematical}}]
\label{lem:app2}
    Let $X_1,\cdots,X_T$ be a random process adaptive to some filtration $\{\mathcal{F}_t\}_{t\le T}$, and $\mathbb{E}_t$ be the conditional expectation on $\mathcal{F}_{t-1}$. Then, for any $\alpha,\delta>0$ we have
    $$\mathrm{Pr}\left[-\sum_{t-1}^T\log \mathbb{E}_t[e^{-\alpha X_t}]\le \alpha\sum_{t=1}^TX_t+\log(1/\delta)\right]\ge 1-\delta.$$
\end{lemma}

\begin{remark}
    It should be noted that the assumption $X_t\ge 0$ is \emph{required} for Lemma~\ref{lem:tightchernoff} to hold. To see this, we group $X^T$ as $X_1X_2, X_3X_4,\cdots$ such that $X_{2t-1}$ is uniform over $\{-1,1\}$ and $X_{2t}=-X_{2t-1}$ for all $t\in [T]$. It is easy to verify that $X_1+\cdots+X_T=0$ almost surely. But $Y_{2t-1}=0$ and $Y_{2t}=-X_{2t-1}$, hence, we have $Y_1+\cdots+Y_T$ is sum of $T/2$ independent uniform distributions over $\{-1,1\}$. Therefore, by the central limit theorem $Y_1+\cdots+Y_T\ge \Omega(\sqrt{T})$ with constant probability. This, unfortunately, limits its application to random variables of form $L(e_t,\hat{p}_t)-L(e_t,f(\x_t))$, such as in Lemma~\ref{lem:exp}. There are, however, special cases such as for log-loss in the \emph{realizable case} that a tight concentration holds for Hellinger divergence, see e.g., Theorem~\ref{thm:highsymetr}.
\end{remark}

\section{Exponential Weighted Average under Exp-concave losses}
\label{sec:ewa}
We now introduce the \emph{Exponential Weighted Average (EWA)} algorithm and its regret analysis under the Exp-concave losses, which is mostly standard~\cite[Chapter 3.3]{lugosi-book} and we include it here for completeness. Let $\mathcal{F}=\{f_1,\cdots,f_K\}\subset \D^{\mathcal{X}}$ be a $\D$-valued function class of size $K$ and $\ell:\tY\times \D\rightarrow \mathbb{R}^{\ge 0}$ be an $\alpha$-Exp-concave loss (see definition in Section~\ref{sec:pre}). The EWA algorithm is presented in Algorithm~\ref{alg:2}.

\begin{algorithm}[h]
\caption{Exponential Weighted Average (EWA) predictor}\label{alg:2}
\textbf{Input}: Class $\mathcal{F}=\{f_1,\cdots,f_K\}$ and $\alpha$-Exp-concave loss $\ell$

Set $\w^1=\{1,\cdots,1\}\in \mathbb{R}^K$;
 
\For{$t=1,\cdots, T$}{
 Receive $\x_t$;

 Make prediction:
 $$\hat{p}_t=\frac{\sum_{k=1}^K\w^t[k]f_k(\x_t)}{\sum_{k=1}^K\w^t[k]}.$$
 
 Receive noisy label $\tilde{y}_t$;
 
 \For{$k\in [K]$}{
    Set $\w^{t+1}[k]=\w^t[k]e^{-\alpha \ell(\tilde{y}_t,f_k(\x_t))}$;
 }
}
\end{algorithm}
Algorithm~\ref{alg:2} provides the following regret bound~\cite[Proposition 3.1]{lugosi-book}.
\begin{proposition}
    \label{prop:ewaregret}
    Let $\mathcal{F}\subset \D^{\mathcal{X}}$ be any finite class of size $K$ and $\ell$ be an $\alpha$-Exp-concave loss. If $\hat{p}_t$ is the predictor in Algorithm~\ref{alg:2}, then for \emph{any} $\x^T\in \mathcal{X}^T$ and $\ty^T\in \tY^T$ we have
    $$\sup_{f\in \mathcal{F}}\sum_{t=1}^T\ell(\ty_t,\hat{p}_t)-\ell(\ty_t,f(\x_t))\le \frac{\log|\mathcal{F}|}{\alpha}.$$
\end{proposition}
\begin{proof}
    We now fix $\x^T$, $\ty^T$ and \emph{any} $f^*\in \mathcal{F}$. Denote $W^t=\sum_{k=1}^K\w^t[k]$. We have
    \begin{align*}
        \frac{W^{t+1}}{W^t}&=\sum_{k=1}^K \frac{\w^t[k]e^{-\alpha\ell(\ty_{t},f_k(\x_{t}))}}{W^t}\\
        &=\sum_{k=1}^K\frac{\w^t[k]}{W^t}e^{-\alpha \ell(\ty_{t},f_k(\x_{t}))}\\
        &\le e^{-\alpha \ell\left(\ty_{t},\sum_{k=1}^K\w^t[k]f_k(\x_{t})/W^t\right)}\\
        &=e^{-\alpha \ell(\ty_t,\hat{p}_t)},
    \end{align*}
    where the inequality follows by Jensen's inequality and definition of $\alpha$-Exp-concavity, and the last equality follows by definition of $\hat{p}_t$. Therefore, by telescoping the product we have
    \begin{align*}
        \log W^{T+1}-\log W^1=\log\frac{W^{T+1}}{W^1}=\log\prod_{t=1}^{T}\frac{W^{t+1}}{W^t}\le -\alpha\sum_{t=1}^T\ell(\ty_t,\hat{p}_t).
    \end{align*}
    Note that $\log W^{T+1}=\log\left(\sum_{k=1}^Ke^{-\alpha\sum_{t=1}^T\ell(\ty_t,f_k(\x_t))}\right)\ge -\alpha\sum_{t=1}^T \ell(\ty_t,f^*(\x_t))$ and $\log W^1=\log K$, we have
    $$\sum_{t=1}^T\ell(\ty_t,\hat{p}_t)-\ell(\ty_t,f^*(\x_t))\le \frac{\log K}{\alpha},$$
    as needed.
\end{proof}

\section{Omitted Proofs in Section~\ref{sec:pre}}
\label{sec:proof2.1}
In this appendix, we present the omitted proofs from Section~\ref{sec:pre}.

\begin{proof}[Proof of Proposition~\ref{prop:bregman}]
    By definition of Bregman divergence, we have
    $$L(p,q_1)-L(p,q_2)=F(q_2)-F(q_1)-p^{\mathsf{T}}(\nabla F(q_1)-\nabla F(q_2))+q_1^{\mathsf{T}}\nabla F(q_1)-q_2^{\mathsf{T}}\nabla F(q_2).$$
    Note that the above expression is a \emph{linear} function w.r.t. $p$. Therefore, by taking expectation over $p\sim P$ and using the linearity of expectation, one can verify the claimed identity holds.
\end{proof}

\begin{proof}[Proof of Proposition~\ref{prop:exp}]
    The $1$-Exp-concavity of log-loss can be verified directly. To prove the $1/4$-Exp-concavity of Brier loss, we have by~\citet[Lemma 4.2]{hazan2016introduction} that a function $f$ is $\alpha$-Exp-concave if and only if
    $$\alpha\nabla f(p)\nabla f(p)^{\mathsf{T}}\preceq \nabla^2 f(p).$$
    For any $q\in \D$, we denote $f(p)=||p-q||_2^2$. We have $\nabla f(p) = 2(p-q)$ and $\nabla^2 f(p) = 2I$, where $I$ is the identity matrix. Taking any $u\in \mathbb{R}^M$, we have $\frac{1}{4}\langle u,2(p-q)\rangle^2\le ||u||_2^2||p-q||_2^2\le 2||u||_2^2=2u^{\mathsf{T}}Iu$, where the first inequality follows by Cauchy-Schwarz inequality and the second inequality follows by:
    $$||p-q||_2^2=\sum_{\ty\in \tY}(p[\ty]-q[\ty])^2\le \sum_{\ty\in \tY}\max\{p[\ty],q[\ty]\}^2\le \sum_{\ty\in \tY}p[\ty]^2+q[\ty]^2\le 2,$$
    since $p,q\in \D$. This completes the proof.
\end{proof}

\section{Proof of Lemma~\ref{lem:exp}}
\label{sec:prooflem34}
Let $\Phi$ be the \emph{Exponentially Weighted Average (EWA)} estimator as in Algorithm~\ref{alg:2} with input $\mathcal{F}$ and loss $\ell(\ty,p)\overset{\mathsf{def}}{=}L(e_{\ty},p)$. Let $\tilde{y}^T$ be any realization of the labels and $e_t$ be the standard base of $\mathbb{R}^M$ with value $1$ at position $\tilde{y}_t$ and zeros otherwise. By $\alpha$-Exp-concavity of loss $\ell$ and the regret bound from Proposition~\ref{prop:ewaregret} (view $\x_t=\psi_t(\tilde{y}^{t-1})$), we have:
    \begin{equation}
        \label{lm8:eq1}
        \sup_{f\in \mathcal{F},\psi^T,\tilde{y}^T\in \tilde{\mathcal{Y}}^T}\sum_{t=1}^TL(e_t,\hat{p}_t)-L(e_t,f(\psi_t(\tilde{y}^{t-1})))\le \frac{\log|\mathcal{F}|}{\alpha},
    \end{equation}
    where $\psi^T=\{\psi_1,\cdots,\psi_T\}$ runs over all functions $\psi_t:\tilde{\mathcal{Y}}^{t-1}\rightarrow \mathcal{X}$ for $t\in [T]$. Note that this bound holds \emph{point-wise} w.r.t. any individual $\psi^T,\tilde{y}^T$.

    Fix any $\psi^{T}$ and distribution $\tilde{p}^T$ over $\tilde{\mathcal{Y}}^T$. We denote $\mathbb{E}_t$ as the conditional expectation on $\tilde{y}_t$ over the randomness of $\tilde{y}^T\sim\tilde{p}^T$ conditioning on $\tilde{y}^{t-1}$ and denote $\tilde{p}_t$ as the \emph{conditional} marginal. By Proposition~\ref{prop:bregman}, we have for all $t\in [T]$ that:
    $$\mathbb{E}_t\left[L(e_t,\hat{p}_t)-L(e_t,f(\psi_t(\tilde{y}^{t-1})))\right]=L(\tilde{p}_t,\hat{p}_t)-L(\tilde{p}_t,f(\psi_t(\tilde{y}^{t-1})))),$$
    since $\mathbb{E}_t[e_t]=\tilde{p}_t$ for $\tilde{y}_t\sim \tilde{p}_t$, $\hat{p}_t$ depending only on $\tilde{y}^{t-1}$ and $L$ is a Bregman divergence. We now take $\mathbb{E}_{\tilde{y}^T}$ on both sides of (\ref{lm8:eq1}). By $\sup\mathbb{E}\le \mathbb{E}\sup$ and the law of total probability (i.e., $\mathbb{E}_{\ty^T}[X_1+\cdots+X_T]=\mathbb{E}_{\ty^T}[\mathbb{E}_1[X_1]+\cdots+\mathbb{E}_T[X_T]]$ for any random variables $X^T$), we have:
    $$\sup_{f\in \mathcal{F}}\sup_{\psi^T,\tilde{p}^T}\mathbb{E}_{\tilde{y}^T\sim \tilde{p}^T}\left[\sum_{t=1}^TL(\tilde{p}_t,\hat{p}_t)-L(\tilde{p}_t,f(\psi_t(\tilde{y}^{t-1})))\right]\le \frac{\log|\mathcal{F}|}{\alpha},$$
    where $\tilde{p}^T$ runs over all distributions over $\tilde{\mathcal{Y}}^T$ and $\psi^T$ runs over all functions $\psi_t:\tilde{\mathcal{Y}}^{t-1}\rightarrow \mathcal{X}$. The lemma then follows by the equivalence between operators $\mathbb{Q}^T\equiv \sup_{\psi^T,\tilde{p}^T}\mathbb{E}_{\tilde{y}^T}$ when taking the kernel $\mathcal{Q}_{y}^{\x}=\mathcal{D}(\tilde{\mathcal{Y}})$ (see the discussion following Definition~\ref{def:expectrisk}). The last part follows by the fact that the exponential weighted average estimator automatically ensures $\hat{p}_t$ is a convex combination of $\{f(\x_t):f\in \mathcal{F}\}$ for all $t\in [T]$.

\section{Proof of Theorem~\ref{thm:hptest}}
\label{sec:proofthm17}
We start with an application of the minimax theorem to hypothesis testing~\footnote{This result was mentioned in~\cite[Chapter 32.2]{pw22}, without providing a proof.}.
\begin{lemma}
\label{lem:minimax}
    Let $\mathcal{P}_0$ and $\mathcal{P}_1$ be two sets of distributions over a finite domain $\Omega$. If $\mathcal{P}_0$ and $\mathcal{P}_1$ are convex under $L_1$ distance (i.e., total variation), then
    $$\min_{\phi~:~\Omega\rightarrow [0,1]}\sup_{p_0\in \mathcal{P}_0,p_1\in \mathcal{P}_1}\left\{\mathbb{E}_{\omega\sim p_0}[1-\phi(\omega)]+\mathbb{E}_{\omega\sim p_1}[\phi(\omega)]\right\}=1-\inf_{p_0\in \mathcal{P}_0,p_1\in \mathcal{P}_1}||p_0-p_1||_{\mathsf{TV}}.$$
    Moreover, if $\phi^*$ is the function attains minimal, then the tester $\psi^*(\omega)=1\{\phi^*(\omega)<0.5\}$ achieves
    $$\sup_{p_0\in \mathcal{P}_0,p_1\in \mathcal{P}_1}\{\mathrm{Pr}_{\omega\sim p_0}[\psi^*(\omega)\not=0]+\mathrm{Pr}_{\omega\sim p_1}[\psi^*(\omega)\not=1]\}\le 2(1-\inf_{p_0\in \mathcal{P}_0,p_1\in \mathcal{P}_1}||p_0-p_1||_{\mathsf{TV}}).$$
\end{lemma}
\begin{proof}
    Observe that the function $\phi$ can be viewed as a vector in $[0,1]^{\Omega}$. Moreover, the distributions over $\Omega$ can be viewed as vectors in $[0,1]^{\Omega}$ as well.  Therefore, we have
    $$\mathbb{E}_{\omega\sim p_0}[1-\phi(\omega)]+\mathbb{E}_{\omega\sim p_1}[\phi(\omega)]=\langle p_0,1-\phi \rangle+\langle p_1,\phi \rangle,$$ which is a linear function w.r.t. both $(p_0,p_1)$ and $\phi$. Since the both $\mathcal{P}_0\times \mathcal{P}_1$ and $[0,1]^{\Omega}$ are convex and $[0,1]^{\Omega}$ is compact, we can invoke the minimax theorem~\cite[Thm 7.1]{lugosi-book} to obtain
    \begin{align*}
        \min_{\phi~:~\Omega\rightarrow [0,1]}&\sup_{p_0\in \mathcal{P}_0,p_1\in \mathcal{P}_1}\{\mathbb{E}_{\omega\sim p_0}[1-\phi(\omega)]+\mathbb{E}_{\omega\sim p_1}[\phi(\omega)]\}\\&\qquad=\sup_{p_0\in \mathcal{P}_0,p_1\in \mathcal{P}_1}\min_{\phi~:~\Omega\rightarrow [0,1]}\left\{\mathbb{E}_{\omega\sim p_0}[1-\phi(\omega)]+\mathbb{E}_{\omega\sim p_1}[\phi(\omega)]\right\}\\
        &\qquad=\sup_{p_0\in \mathcal{P}_0,p_1\in \mathcal{P}_1}\{1-||p_0-p_1||_{\mathsf{TV}}\},
    \end{align*}
    where the last equality follows by Le Cam's two point lemma~\cite[Theorem 7.7]{pw22}. Let $\phi^*$ be the function attains minimal and $\psi^*(\omega)=1\{\phi^*(\omega)<0.5\}$. We have $1\{\psi^*(\omega)\not=i\}\le 2(1-i-\phi^*(\omega))$ for all $i\in \{0,1\}$. To see this, for $i=0$, we have $\psi^*(\omega)\not=0$ only if $\phi^*(\omega)<0.5$, thus $1-\phi^*(\omega)\ge 0.5$ (the case for $i=1$ follows similarly). Therefore, we have for all $p_0\in \mathcal{P}_0,p_1\in \mathcal{P}_1$
    $$\mathrm{Pr}_{\omega\sim p_0}[\psi^*(\omega)\not=0]+\mathrm{Pr}_{\omega\sim p_1}[\psi^*(\omega)\not=1]\le 2(\mathbb{E}_{\omega\sim p_0}[1-\phi^*(\omega)]+\mathbb{E}_{\omega\sim p_1}[\phi^*(\omega)]).$$This completes the proof.
\end{proof}
 We have the following key property:
\begin{lemma}
    Let $\mathcal{Q}_1^J$ and $\mathcal{Q}_2^J$ be the sets in Theorem~\ref{thm:hptest}. Then $\mathcal{Q}_1^J$ and $\mathcal{Q}_2^J$ are convex.
\end{lemma}
\begin{proof}
    Let $p_1,p_2\in \mathcal{Q}_i^J$ for $i\in \{1,2\}$ and $\lambda\in [0,1]$. We need to show that $p=\lambda p_1+(1-\lambda)p_2\in \mathcal{Q}_i^J$ as well. For any given $t\in [J]$ and $\ty^{t-1}\in \tY^{t-1}$,  we have
    \begin{align*}
        p(\ty_t\mid \ty^{t-1})&=\frac{\lambda p_1(\ty^t)+(1-\lambda)p_2(\ty^t)}{\lambda p_1(\ty^{t-1})+(1-\lambda)p_2(\ty^{t-1})}\\&=\lambda \frac{p_1(\ty^{t-1})}{p(\ty^{t-1})}p_1(\ty_t\mid\ty^{t-1})+(1-\lambda)\frac{p_2(\ty^{t-1})}{p(\ty^{t-1})}p_2(\ty_t\mid\ty^{t-1})\in \mathcal{Q}_i^{\ty^{t-1}}
    \end{align*}
    where the last inclusion follows by convexity of $\mathcal{Q}_i^{\ty^{t-1}}$ as assumed in Theorem~\ref{thm:hptest}. Therefore, we have $p\in \mathcal{Q}_i^J$ by definition of $\mathcal{Q}_i$.
\end{proof}

Now, our main technical part is to bound the total variation \(\mathsf{TV}(\mathcal{Q}_1^J, \mathcal{Q}_2^J)\). The primary challenge comes from controlling the dependencies of conditional marginals of the distributions. To proceed, we now introduce the concept of \emph{Renyi divergence}. Let $p_1,p_2$ be two distributions over the same finite domain $\Omega$, the $\alpha$-Renyi divergence is defined as
$$D_{\alpha}(p_1,p_2)=\frac{1}{\alpha-1}\log \mathbb{E}_{\omega\sim p_2}\left[\left(\frac{p_1(\omega)}{p_2(\omega)}\right)^{\alpha}\right].$$
If $p,q$ are distributions over domain $\Omega_1\times \Omega_2$ and $r$ is a distribution over $\Omega_1$, then the \emph{conditional} $\alpha$-Renyi divergence is defined as
$$D_{\alpha}(p,q\mid r)=\frac{1}{\alpha-1}\log \mathbb{E}_{\omega_1\sim r}\left[\sum_{\omega_2\in \Omega_2}p(\omega_2\mid \omega_1)^{\alpha}q(\omega_2\mid \omega_1)^{1-\alpha}\right].$$

The following property about Renyi divergence is well known~\cite[Chapter 7.12]{pw22}:
\begin{lemma}
\label{lem:chainrule}
    Let $p,q$ be two distributions over $\Omega_1\times \Omega_2$ and $p^{(1)}$ and $q^{(1)}$ be the restrictions of $p,q$ on $\Omega_1$, respectively. Then the following chain rule holds
    $$D_{\alpha}(p,q)=D_{\alpha}(p^{(1)},q^{(1)})+D_{\alpha}(p,q\mid r),$$
    where $r(\omega_1)=p^{(1)}(\omega_1)^{\alpha}q^{(1)}(\omega_1)^{1-\alpha}e^{-(\alpha-1)D_{\alpha}(p^{(1)},q^{(1)})}$ is a distribution over $\Omega_1$.
\end{lemma}

We now arrive at our main technical result for bounding the Renyi divergence between $\mathcal{Q}_1^J$ and $\mathcal{Q}_2^J$ in Theorem~\ref{thm:hptest}:
\begin{proposition}
\label{prop:tensorrenyi}
    Let $\mathcal{Q}_1^J$ and $\mathcal{Q}_2^J$ be the sets in Theorem~\ref{thm:hptest}. If for all $t\in [J]$ and $\ty^{t-1}\in \tY^{t-1}$, we have $\inf_{p_t\in \mathcal{Q}_1^{\ty^{t-1}},q_t\in \mathcal{Q}_2^{\ty^{t-1}}}D_{\alpha}(p_t,q_t)\ge \eta_t$ for some $\eta_t\ge 0$. Then
    $$\inf_{p\in \mathcal{Q}_1^J,q\in \mathcal{Q}_2^J}D_{\alpha}(p,q)\ge \sum_{t=1}^J\eta_t.$$
\end{proposition}
\begin{proof}
    We prove by induction on $J$. The base case for $J=1$ is trivial. We now prove the induction step with $J\ge 2$. For any pair $p\in \mathcal{Q}_1^J$ and $q\in \mathcal{Q}_2^J$, we have by Lemma~\ref{lem:chainrule} that
    $D_{\alpha}(p,q)=D_{\alpha}(p^{(1)},q^{(1)})+D_{\alpha}(p,q\mid r)$, where $p^{(1)}$, $q^{(1)}$ are restrictions of $p$, $q$ on $\ty^{J-1}$ and $r$ is a distribution over $\tY^{J-1}$. By definition of $\alpha$-Renyi divergence, we have
    \begin{align*}
        D_{\alpha}(p,q\mid r)&\ge \inf_{\ty^{J-1}}\frac{1}{\alpha-1}\log \sum_{\ty_J\in \tY}p(\ty_J\mid \ty^{J-1})^{\alpha}q(\ty_J\mid \ty^{J-1})^{1-\alpha}\\
        &=\inf_{\ty^{J-1}} D_{\alpha}(p_{\ty_J\mid \ty^{J-1}},q_{\ty_J\mid \ty^{J-1}})\\
        &\overset{(a)}{\ge} \inf_{p\in \mathcal{Q}_1^{\tilde{y}^J},q\in \mathcal{Q}_2^{\tilde{y}^J}}D_{\alpha}(p,q)\overset{(b)}{\ge} \eta_J,
    \end{align*}
    where $(a)$ follows since $p_{\ty_J\mid \ty^{J-1}}\in \mathcal{Q}_1^{\tilde{y}^J}$ and $q_{\ty_J\mid \ty^{J-1}}\in \mathcal{Q}_2^{\tilde{y}^J}$ by the definition of $\mathcal{Q}_1^J$ and $\mathcal{Q}_2^J$; $(b)$ follows by assumption. The result then follows by induction hypothesis $D_{\alpha}(p^{(1)},q^{(1)})\ge \sum_{t=1}^{J-1}\eta_t$, since $p^{(1)}\in \mathcal{Q}_1^{J-1}$ and $q^{(1)}\in \mathcal{Q}_2^{J-1}$.
\end{proof}

The following result converts the Renyi divergence based bounds to that with Hellinger divergence.
\begin{proposition}
\label{prop:tensorH}
    Let $\mathcal{Q}_1^J$ and $\mathcal{Q}_2^J$ be the sets in Theorem~\ref{thm:hptest}. If for all $t\in [J]$ and $\ty^{t-1}\in \tY^{t-1}$, we have $H^2(\mathcal{Q}_1^{\ty^{t-1}},\mathcal{Q}_2^{\ty^{t-1}})\ge \gamma_t$ for some $\gamma_t\ge 0$. Then
    $$\inf_{p\in \mathcal{Q}_1^J,q\in \mathcal{Q}_2^J}H^2(p,q)\ge 2\left(1-\prod_{t=1}^J(1-\gamma_t/2)\right).$$
\end{proposition}
\begin{proof} Observe that, for any distributions $p,q$ we have
\begin{equation}
\label{eq:h2r}
    H^2(p,q)=2(1-e^{-\frac{1}{2}D_{1/2}(p,q)}).
\end{equation}
Specifically, for give $p\in \mathcal{Q}_1^J$ and $q\in \mathcal{Q}_2^J$, we have $$1-H^2(p,q)/2=e^{-\frac{1}{2}D_{1/2}(p,q)}\le e^{-\frac{1}{2}\sum_{t=1}^J\eta_t}=\prod_{t=1}^Je^{-\frac{1}{2}\eta_t}\le \prod_{t=1}^J(1-\gamma_t/2),$$ where $\eta_t$s are the constants in Proposition~\ref{prop:tensorrenyi} and the last inequality follows by $e^{-\frac{1}{2}\eta_t}\le 1-\gamma_t/2$ due to (\ref{eq:h2r}) again. This completes the proof.
\end{proof}

\begin{proof}[Proof of Theorem~\ref{thm:hptest}]
    We have by Lemma~\ref{lem:minimax} that the testing error is upper bounded by $1-\inf_{p\in \mathcal{Q}_1,q\in \mathcal{Q}_2}||p-q||_{\mathsf{TV}}$. Fix any pair $p,q$, we have by relation between Hellinger and total variation that $1-||p-q||_{\mathsf{TV}}\le 1-\frac{1}{2}H^2(p,q)$. The result follows by Proposition~\ref{prop:tensorH}.
\end{proof}
\section{ Proof of High Probability Minimax Risk of Theorem~\ref{thm:symetric}}
\label{sec:apphigh}

We begin with the following key inequality:
\begin{lemma}
\label{lem:app1}
    Let $\tilde{p}=(1-\eta')e_{\ty}+\eta' u$, $p=(1-\eta)e_{\ty}+\eta u$ and $\hat{p}=(1-\eta)a+\eta u$, where $e_{\ty}, a,u\in \D$ and $0\le \eta'\le \eta <1$, such that $e_{\ty}$ is the distribution assigning probability $1$ on $\ty$, $u$ is uniform over $\tY$ and $a\in \D$ is arbitrary. Then
    \begin{equation}
    \label{app1:eq1}
        \sum_{\ty'\in \tY}\tilde{p}[\ty']\sqrt{\frac{\hat{p}[\ty']}{p[\ty']}}\le \sum_{\ty'\in \tY}p[\ty']\sqrt{\frac{\hat{p}[\ty']}{p[\ty']}}=\sum_{\ty'\in \tY}\sqrt{p[\ty']\hat{p}[\ty']}.
    \end{equation}
\end{lemma}
\begin{proof}
    Denote $|\mathcal{\tY}|=M$, and let $r\in \mathbb{R}^{\tY}$ be the vector such that $r[\ty']=\sqrt{\hat{p}[\ty']/p[\ty']}$. We have the LHS of (\ref{app1:eq1}) equals $e_{\ty}^{\mathsf{T}}r + \eta'(u-e_{\ty})^{\mathsf{T}}r$. We claim that $f(\eta')\overset{\mathsf{def}}{=}e_{\ty}^{\mathsf{T}}r + \eta'(u-e_{\ty})^{\mathsf{T}}r$ attains maximum when $\eta'=\eta$, which will finish the proof. It is sufficient to prove that $(u-e_{\ty})^{\mathsf{T}}r\ge 0$ since $f(\eta')$ is a linear function w.r.t. $\eta'$. We have
    $$u^{\mathsf{T}}r=\frac{1}{M}\sum_{\ty'\in \tY}\sqrt{\frac{\hat{p}[\ty']}{p[\ty']}},~e_{\ty}^{\mathsf{T}}r=\sqrt{\frac{\hat{p}[\ty]}{p[\ty]}}.$$
    We only need to show that $\forall \ty'\in \tY$ with $\ty'\not=\ty$, we have $\sqrt{\hat{p}[\ty']/p[\ty']}\ge \sqrt{\hat{p}[\ty]/p[\ty]}$, i.e.,
    $$\frac{p[\ty]}{p[\ty']}\ge \frac{\hat{p}[\ty]}{\hat{p}[\ty']}.$$
    Note that, $p[\ty]=1-\eta+\frac{\eta}{M}$, $p[\ty']=\frac{\eta}{M}$, $\hat{p}[\ty]=(1-\eta)a[\ty]+\frac{\eta}{M}$ and $\hat{p}[\ty']=(1-\eta)a[\ty']+\frac{\eta}{M}$, i.e., we have $p[\ty]\ge \hat{p}[\ty],\hat{p}[\ty']\ge p[\ty]$. The result now follows by the simple fact that for \emph{any} $a\ge b,c\ge d\ge 0$ we have $\frac{a}{d}\ge\frac{b}{c}$.
\end{proof}

We are now ready to state our main result of this appendix, which establishes the high probability bounds in Theorem~\ref{thm:symetric}. 
\begin{theorem}
\label{thm:highsymetr}
    Let $\mathcal{H}\subset \mathcal{Y}^{\mathcal{X}}$ be any finite class and $\mathcal{K}^{\eta}$ be the kernel in Section~\ref{sec:special} with $0\le \eta<1$. Then, the \emph{high probability} minimax risk at confidence $\delta$ is upper bounded by
    $$B^{\delta}(\mathcal{H},\mathsf{P},\mathcal{K}^{\eta})\le \frac{\log|\mathcal{H}|+2\log(1/\delta)}{(1-\eta)^2/4}.$$
    Furthermore, for $1-\eta\ll \frac{1}{M}$ we have $B^{\delta}(\mathcal{H},\mathsf{P},\mathcal{K}^{\eta})\le O\left(\frac{\log|\mathcal{H}|+\log(1/\delta)}{M(1-\eta)^2}\right)$. 
\end{theorem}
\begin{proof}
    Let $\mathcal{F}$ be the class as in the proof of Theorem~\ref{thm:symetric} and $\hat{p}_t$ be the \emph{Exponential Weighted Average} algorithm under Log-loss, see Algorithm~\ref{alg:2}. We have by Proposition~\ref{prop:ewaregret} that for \emph{any} $\ty^T\in \tY^T$
    $$\sup_{\x^T\in \mathcal{X}^T}\sum_{t=1}^T\log\frac{f^*(\x_t)[\ty_t]}{\hat{p}_t[\ty_t]}\le \log|\mathcal{F}|$$
    where $f^*$ is the corresponding function of the underlying truth
    $h^*\in \mathcal{H}$ (see the proof of Theorem~\ref{thm:symetric}). We now assume $\ty^T$ are sampled from $\tilde{p}^T$, where $\tilde{p}^T$ are the noisy label distributions selected by the adversary. Denote by $\mathbb{E}_t$  the conditional expectation on $\ty^{t-1}$. We have
    $$\mathbb{E}_t\left[e^{-\frac{1}{2}\log\frac{f^*(\x_t)[\ty_t]}{\hat{p}_t[\ty_t]}}\right]=\mathbb{E}_{\ty_t\sim \tilde{p}_t}\sqrt{\frac{\hat{p}[\ty_t]}{f^*(\x_t)[\ty_t]}}\le \sum_{\ty_t\in \tY}\sqrt{\hat{p}[\ty_t]f^*(\x_t)[\ty_t]},$$
    where the inequality follows from Lemma~\ref{lem:app1}. By a similar argument as in the proof of~\cite[Lemma A.14]{foster2021statistical}, we have
    $$\log \sum_{\ty_t\in \tY}\sqrt{\hat{p}[\ty_t]f^*(\x_t)[\ty_t]}=\log \left(1-\frac{1}{2}H^2(\hat{p}_t,f^*(\x_t))\right)\le -\frac{1}{2}H^2(\hat{p}_t,f^*(\x_t)),$$
    where the first equality follows by definition of squared Hellinger divergence. Taking $X_t=\log\frac{f^*(\x_t)[\ty_t]}{\hat{p}_t[\ty_t]}$, $\alpha=\frac{1}{2}$ and invoking Lemma~\ref{lem:app2} we have w.p. $\ge 1-\delta$
    $$\mathrm{Pr}\left[\sum_{t=1}^TH^2(\hat{p}_t,f^*(\x_t))\le \log|\mathcal{F}|+2\log(1/\delta)\right]\ge 1-\delta.$$
    Let now $\hat{y}_t=\arg\max_{\ty}\{\hat{p}_t[\ty]:\ty\in \tY\}$. We have, if $\hat{y}_t\not=h^*(\x_t)$
    $$H^2(\hat{p}_t,f^*(\x_t))\ge ||\hat{p}_t-f(\x_t)||_1^2/4\ge (1-\eta)^2/4,$$
    where the first inequality follows from $\sqrt{H^2(p,q)}\ge ||p-q||_1/2$~\cite[Equation 7.20]{pw22} and the second inequality follows the same as in the proof of Claim~\ref{claim1}. Since $H^2(p,q)\ge 0$ for all $p,q$, we have w.p. $\ge 1-\delta$ that
    $$\sum_{t=1}^T1\{\hat{y}_t\not=h^*(\x_t)\}\le \frac{\log|\mathcal{H}|+2\log(1/\delta)}{(1-\eta)^2/4}.$$
    To prove the second part, we observe that if $\hat{y}_t\not=h^*(\x_t)$, then $\hat{p}_t=(1-\eta)a_t+\eta u$ such that $a_t[h^*(\x_t)]\le \frac{1}{2}$. Since $f^*(\x_t)=(1-\eta)e_{h^*(\x_t)}+\eta u$, we have by direct computation that
    $$H^2(\hat{p}_t,f^*(\x_t))\ge \left(\sqrt{(1-\eta)/2+\frac{\eta}{M}}-\sqrt{1-\eta+\frac{\eta}{M}}\right)^2\sim \frac{M(1-\eta)^2}{16},$$
    where the last asymptote follows by Taylor expansion $\frac{M(\eta-1)^2}{16}+O(\sum_{n=3}^\infty M^{n-1}(1-\eta)^n)$ and the remainder term converges when $1-\eta\ll \frac{1}{M}$.
\end{proof}
\begin{remark}
    Note that, Lemma~\ref{lem:app1} is the key that allows us to reduce our \emph{miss-specified} setting to the well-specified case, such as~\cite[Lemma A.14]{foster2021statistical}, for which a reduction to the Hellinger divergence is possible.
\end{remark}

\bibliography{ml.bib}

\end{document}